\documentclass[lettersize,journal]{IEEEtran}

\usepackage{enumitem}
\usepackage[numbers,sort&compress]{natbib}
\usepackage{url}
\usepackage{amsmath,amssymb,amsfonts}
\newcommand\numberthis{\addtocounter{equation}{1}\tag{\theequation}}
\usepackage{mathtools}
\usepackage{bm}
\usepackage{soul}
\usepackage{stmaryrd}
\usepackage{algorithmic}
\usepackage{graphicx}
\usepackage{multirow}
\usepackage{subcaption}
\usepackage{dsfont}
\usepackage{textcomp}
\usepackage[usenames,dvipsnames]{xcolor}
\def\BibTeX{{\rm B\kern-.05em{\sc i\kern-.025em b}\kern-.08em
T\kern-.1667em\lower.7ex\hbox{E}\kern-.125emX}}
\usepackage{amsthm}
\theoremstyle{definition}
\newcounter{Cequ}

\usepackage{xr-hyper}
\usepackage{hyperref}    
\usepackage{lipsum}
\usepackage{scalerel}
\usepackage{stackengine,wasysym}
\usepackage{cleveref}
\usepackage{sistyle}
\usepackage{arydshln}
\SIthousandsep{,}
\usepackage{setspace}
\usepackage{hyperref}
\usepackage{marginnote}
\usepackage{tabularray}
\UseTblrLibrary{booktabs}

\input{mysymbol.sty}
\newtheorem{problem}{Problem}

\usepackage{kantlipsum}
\usepackage[breakable, theorems, skins]{tcolorbox}
\tcbset{enhanced}

\makeatletter
\newcounter{qrr@oldeq}
\newcounter{qrr@oldsubeq}
\newcounter{qrr@realeq}
\renewenvironment{subequations}{%
  \refstepcounter{equation}%
  \protected@edef\theparentequation{\theequation}%
  \setcounter{parentequation}{\value{equation}}%
  \setcounter{equation}{0}%
  \def\theequation{\theparentequation\alph{equation}}%
  \ignorespaces
}{%
  \setcounter{qrr@oldeq}{\value{parentequation}}%
  \setcounter{qrr@oldsubeq}{\value{equation}}%
  \setcounter{equation}{\value{parentequation}}%
  \ignorespacesafterend
}
\newenvironment{subequations*}{%
  \setcounter{qrr@realeq}{\value{equation}}%
  \let\theparentequation\theequation%
  \patchcmd{\theparentequation}{equation}{parentequation}{}{}%
  \setcounter{parentequation}{\numexpr\value{qrr@oldeq}-1}%
  \setcounter{equation}{\value{qrr@oldsubeq}}%
  \def\theequation{\theparentequation\alph{equation}}%
  \refstepcounter{parentequation}%
  \ignorespaces
}{%
  \setcounter{qrr@oldeq}{\value{parentequation}}%
  \setcounter{qrr@oldsubeq}{\value{equation}}%
  \setcounter{equation}{\value{qrr@realeq}}%
  \ignorespacesafterend
}
\makeatother

\makeatletter
\patchcmd{\@maketitle}
  {\addvspace{0.5\baselineskip}\egroup}
  {\addvspace{-1\baselineskip}\egroup}
  {}
  {}
\makeatother

\begin{document}

\title{
Learning to Transmit with Provable Guarantees in Wireless Federated Learning

\thanks{ 
The research was sponsored by the Army Research Office and was accomplished under Cooperative Agreement Number W911NF-19-2-0269. 
The views and conclusions contained in this document are those of the authors and should not be interpreted as representing the official policies, either expressed or implied, of the Army Research Office or the U.S. Government. 
The U.S. Government is authorized to reproduce and distribute reprints for Government purposes notwithstanding any copyright notation herein.
Preliminary results were presented at~\cite{li2022power}.
\newline
E-mails: \{boning.li, segarra\}@rice.edu, \{jake.b.perazzone.civ, ananthram.swami.civ\}@army.mil.}}

\author{
Boning Li$^\star$,  
Jake Perazzone$^\dag$, 
Ananthram Swami$^\dag$,
and Santiago Segarra$^\star$ \\
\textit{$^\star$Rice University, USA \hspace{1cm}  $^\dag$US DEVCOM Army Research Lab., USA}
}

\maketitle
\begin{abstract}
We propose a novel data-driven approach to allocate transmit power for federated learning (FL) over interference-limited wireless networks.
The proposed method is useful in challenging scenarios where the wireless channel is changing during the FL training process and when the training data are not independent and identically distributed (non-i.i.d.) on the local devices.
Intuitively, the power policy is designed to optimize the information received at the server end during the FL process under communication constraints.
Ultimately, our goal is to improve the accuracy and efficiency of the global FL model being trained.
The proposed power allocation policy is parameterized using \revminor{graph convolutional networks (GCNs),} and the associated constrained optimization problem is solved through a primal-dual (PD) algorithm.
Theoretically, we show that the formulated problem has a zero duality gap and, once the power policy is parameterized, optimality depends on how expressive this parameterization is.
Numerically, we demonstrate that the proposed method outperforms existing baselines under different wireless channel settings and varying degrees of data heterogeneity.
\end{abstract}


\vspace{-1em}
\section{Introduction}\label{s:intro}
Federated learning (FL) enables collaborative training involving multiple workers and one (or more) server(s)~\cite{mcmahan2017communication}. 
In an iterative manner, a global model is first broadcast to all workers and then updated by integrating local models from workers that participate in the current iteration.
Those local models are trained exclusively using the on-device datasets of the corresponding workers. 
Since only model parameters are shared and no explicit transmission of raw data ever occurs, FL is particularly favorable in settings where privacy concerns or policy constraints limit or restrict data sharing, such as in autonomous driving~\cite{li2021privacy} and digital healthcare~\cite{xu2021federated}. 
Given the iterative nature of FL, a large number of local model integration steps may be necessary until the convergence of the global model.
In many studies~\cite{zhang2022federated,shahid2021communication,chen2021communication}, it is shown that communications between the workers and the server are a major bottleneck of FL performance.
While traditional wired communication is being replaced with fast and reliable wireless connections as in the 5G mobile network, smart home devices, and unmanned aerial systems, the growing computational power of edge devices has increased interest in studying FL scenarios where the data is located at the edge of wireless  networks~\cite{tran2019federated,wang2019adaptive,dong2020communication,girgis2021shuffled,hanna2021quantization}.

The specific demands of FL motivate many novel problems in the management of wireless networks.
For example, one may seek to perform power allocation to achieve faster convergence of the global model, where a smart power policy is essential to ensuring the quality of wireless services~\cite{chiang2008power,choi2016power,matthiesen2020globally,chen2020joint,xia2020multi,yoshida2020mab}.
To give a more specific instance,~\cite{chen2020joint} proposed a joint optimization scheme of power, worker, and resource management, \revminor{taking into account} the impact of wireless factors on the FL convergence rate. 
However, only static channels were considered, making it rather impractical for real-world deployment where the required duration of the FL process may well exceed the channel coherence time.
Without knowing the wireless channels a priori,~\cite{xia2020multi} and~\cite{yoshida2020mab} have leveraged multi-armed bandit worker scheduling to reduce FL time consumptions in the presence of heterogeneous data and fluctuating resources, respectively.
However, their applicability is limited due to the assumption of orthogonal channels, which over-simplifies the power optimization in most realistic scenarios.

\revmajor{Outside of the FL context, power allocation problems usually seek to maximize a problem-specific objective subject to simple constraints~\cite{amiri2018machine,zhang2021scalable,nasir2019multi,kumar2021adaptive}.
The objective could be sum-rate, energy efficiency, or proportional fairness, with typical constraints being a box constraint on allocated power.
Challenging power allocation problems are often formulated in interference-limited networks.
These networks differ from orthogonal channels, where signals are isolated to prevent interference. 
Instead, they involve non-orthogonal channels sharing the same time and frequency resources, potentially improving capacity and reducing latency but compromising link reliability~\cite{zhao2016noma}.
Mathematically, this introduces NP-hard optimization challenges, as the aforementioned objectives adopt sum-of-fractions forms due to the incorporation of interference~\cite{salaun2018optimal,freund2001solving}.
}
Since the globally optimal solutions are practically infeasible for moderately-sized networks (tens of workers)~\cite{matthiesen2018optimization}, approximate optimization and statistical learning methods have been exploited to achieve efficient and near-optimal power policies. 
For instance,~\cite{fang2016energy,zhou2016energy} appealed to matching theory,~\cite{dong2014optimal,fang2019optimal} to Lagrangian dual decomposition, and~\cite{yang2017unified,su2020energy,razaviyayn2014successive} to successive concave approximation. 
\revminor{Turning to another set of methods}, machine learning -- in particular, graph neural networks (GNNs) -- has been recently applied to solve power allocation problems~\cite{li2022graph,zhao2021distributed,zhao2022link,eisen2020optimal,chowdhury2021unfolding,chowdhury2021ml}.
For example, to maximize sum-rate in interference networks, a suboptimal power allocation algorithm was enhanced with augmented parameters learned by graph convolutional networks (GCNs) for both single-input single-output~\cite{chowdhury2021unfolding} and multiple-input multiple-output~\cite{chowdhury2021ml} systems. 
Compared to this body of research, FL \revminor{demands extra, often \emph{non-convex} constraints, like} imposing limits on the time and energy consumed by uplink transmissions.

In light of this concern, the primal-dual (PD) algorithm \revminor{is often} considered for such optimization under non-convex constraints in \revminor{a wide range of} wireless~\cite{eisen2019learning} and safe reinforcement learning~\cite{paternain2022safe} settings. 
\revmajor{
Drawing inspiration from this class of research, we propose to tackle the wireless FL-oriented power allocation problem with a learned PD approach through GCNs.  
Our proposed approach is rooted in a PD framework, inheriting its advantages such as the feasibility through penalty terms, the convexity of dual problems, and the iterative resolution of less complex subproblems.
\revminor{At the same time}, we harness the power of a data-driven approach by incorporating the topological nature of wireless network data in our neural architecture.
To the best of our knowledge, this is the first GCN-based power policy optimized by the PD method for FL in interference wireless networks. 
We provide theoretical guarantees for the proposed method and also demonstrate its state-of-the-art performance in multiple independent and identically distributed (i.i.d.) and non-i.i.d. scenarios.
Ultimately, optimizing transmit power can enhance the global FL model's performance, since there is less transmission noise, higher data availability, and enhanced communication robustness during the collaborative training process.}

\vspace{2mm}
\noindent {\bf Contributions.} 
The main contributions of this work are:
\begin{itemize}[topsep=0pt, wide=0pt]
    \item[i)] We formulate the power allocation problem under FL requirements as a maximization problem under non-convex conditional expectation constraints, to which we propose a solution based on a GCN-parameterized PD approach.\looseness=-1
    \item[ii)] We prove that our approach has a zero duality gap under mild conditions and universal parameterization.\looseness=-1
    \item[iii)] Through numerical experiments on simulated wireless channels and real-world datasets, we validate the effectiveness of the proposed FL power allocation approach in both i.i.d. and non-i.i.d. cases.\looseness=-1
\end{itemize}

\vspace{2mm}
\noindent {\bf Paper outline.} 
In Section~\ref{s:sys}, we present the system model and formulate the FL power allocation as a non-convex constrained optimization problem. 
In Section~\ref{s:alg}, we give an in-depth description of the proposed method along with theoretical proofs of its zero duality gap and of the parameterization suboptimality.
In Section~\ref{s:exp}, we discuss numerical experiments and show that the proposed method always outperforms heuristic and topology-agnostic power allocation benchmarks. 
Finally, we close this paper with conclusions in Section~\ref{s:end}.

\vspace{2mm}
\noindent {\bf Notation.} 
Sets and maps are typeset in calligraphic letters ($\ccalA$). 
Vectors are in bold lower case (\revminor{e.g.,} $\bba$) and matrices in bold upper case (\revminor{e.g.,} $\bbA$).
Alternative notations for sets $\{a_i\}^n_{i=1}{\,=\,}\{a_1, ..., a_n\}$ and vectors $[a_i]^n_{i=1}{\,=\,}[a_1, ..., a_n]$ may be present for clarification.
We use $\diag({\bbA})$ to denote a diagonal matrix storing diagonal elements of matrix ${\bbA}$. 
Additionally, ${\bbzero}$ and ${\bf 1}$ denote all-zeros and all-ones vectors of appropriate size, respectively.
When defining a function, its parameters may follow the input variables and be separated by a semicolon, e.g., $f(\bbA,\bbb;\bbw)$ denotes a function $f$ parameterized by $\bbw$.
For conditional statements, we define the indicator function $\mathds{1}(s){\,=\,}1$ if the condition $s$ is true, 0 otherwise.
\revmajor{To project a value $x$ onto a range $\ccalX$, we represent this operation as $\text{proj}_\ccalX(x)$. 
A simplified notation $[x]_+{\,=\,}\text{proj}_{\mbR_+}(x)$ is adopted for the positive part function, equivalent to $\max(x,0)$.}
When applied to vectors, scalar functions or operators are applied in an elementwise manner.
\revmajor{
Above, we have presented the general notation rules applied in this paper. 
For further references to important variables related to the physical processes under study and the problems formulated, please refer to Table~\ref{tab:notation} in Appendix~\ref{ap:notations}. 
We will explain additional symbols introduced during proofs and experiments inline as they come into play.
}

\section{System Model}\label{s:sys}
We formally introduce the FL system model and the interference-limited wireless channel model in Sections~\ref{ss:fl} and~\ref{ss:wl}, respectively. 
We provide a precise formulation of power allocation for FL as an optimization problem in Section~\ref{ss:prob}.

\vspace{-1em}
\subsection{Federated Learning over Wireless Networks}\label{ss:fl}

As depicted in Fig.~\ref{ff:flsys}, our wireless FL system consists of a single server and $L$ mobile worker devices, in which the server and workers communicate through wireless links.

\begin{figure}[t]
\centering
  \includegraphics[width=8cm,
    trim= 0cm 3.5cm 17cm .5cm,
    ]{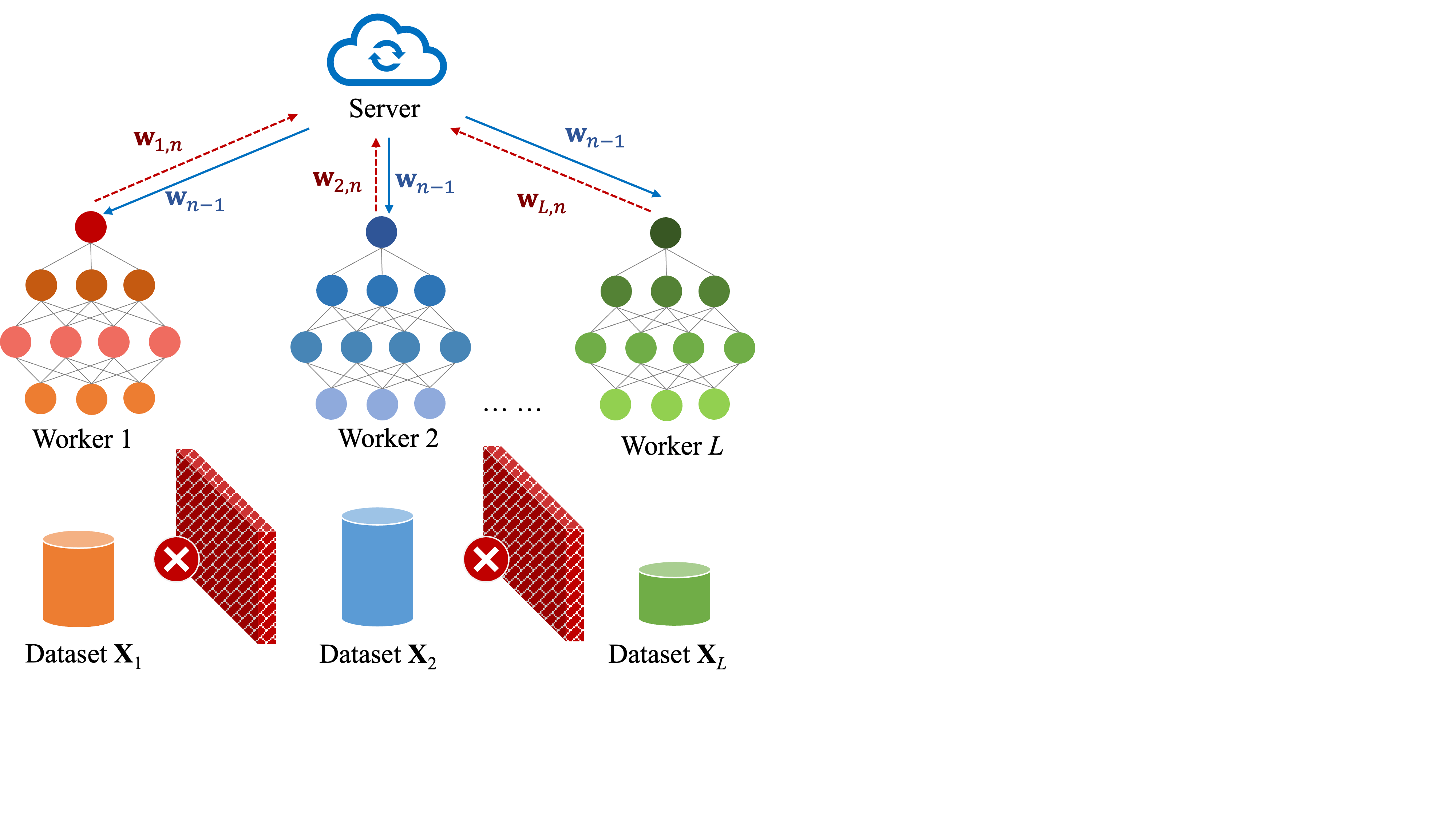}
  \captionof{figure}{
  A schematic view of the FL system model. 
  \revminor{Local datasets within the respective workers are not shared with other workers or the server, maintaining data privacy and security.}
  }  \label{ff:flsys}
  \vspace*{-1em}
\end{figure}

Each worker $i$ collects local data $\bbX_i{=}[\bbx_i^{(1)},...,\bbx_i^{(k_i)}]$,
where $k_i$ is the number of data samples in $i$. 
Subsequently, the total number of training data samples is $K{=}\sum^{L}_{i=1}{k_i}$ and the complete dataset $\ccalX{=}\{\bbX_1,...,\bbX_L\}$. 
We will assume that $\ccalX$ is i.i.d. across all workers unless explicitly stated otherwise.
Much like any other machine learning problems, FL takes the general form of minimizing an objective function but for all workers: $\min\limits_{\bbw}\sum_i^L f_i(\bbw)$, where $ f_i(\bbw){\,=\,}\ell(\bbX_i;\bbw)$ denotes a {\it local} loss, such as quadratic loss for regression tasks and cross-entropy for classification tasks.
Workers start each round of local optimization with the latest global model.
The $n^\text{th}$ FL iteration, for example, begins with the following assignment: $\bbw_{i,n}{\,=\,}\bbw_{n-1}$, where $\bbw_{i,n}$ denotes the model to be updated at worker $i$ and $\bbw_{n-1}$ the global model resulting from the last iteration.
\revminor{Following this,} each worker updates their parameters $\bbw_{i,n}$ based on local data and uploads these to the server, and the server aggregates those into a global one via $\bbw_n{\,=\,}\frac{1}{K}\sum_{i=1}^L k_i \bbw_{i,n}$, effectively ending the $n^\text{th}$ FL iteration.

The wireless setting adds additional complexities to the basic FL setup.
Since workers upload local models over wireless links that could be unreliable in realistic environments, the transmission may be corrupted. 
If no error-detecting codes, e.g., cyclic redundancy check (CRC) codes, are used to eliminate updates that are in error, erroneous transmissions may have a severe impact on FL performance because errors will be aggregated into the global model and contaminate all workers in the next iteration. 
Even with error detection, frequent transmission errors are still \revminor{undesirable} because we would have to abandon too many updates, which is at least a waste of resources.
Additionally, excessive transmission time from only a few workers can significantly slow down the global convergence of the entire system.

\vspace*{-1em}
\subsection{Multiuser Interference Network}\label{ss:wl}

We consider the uplink of a network with $L$ single-antenna mobile workers and an $n_R$-antenna central server or base station (BS); see Fig.~\ref{ff:wlsys}.

\begin{figure}[t]
\centering
  \includegraphics[width=5.cm,
    trim= 0 0 0 0cm
    ]{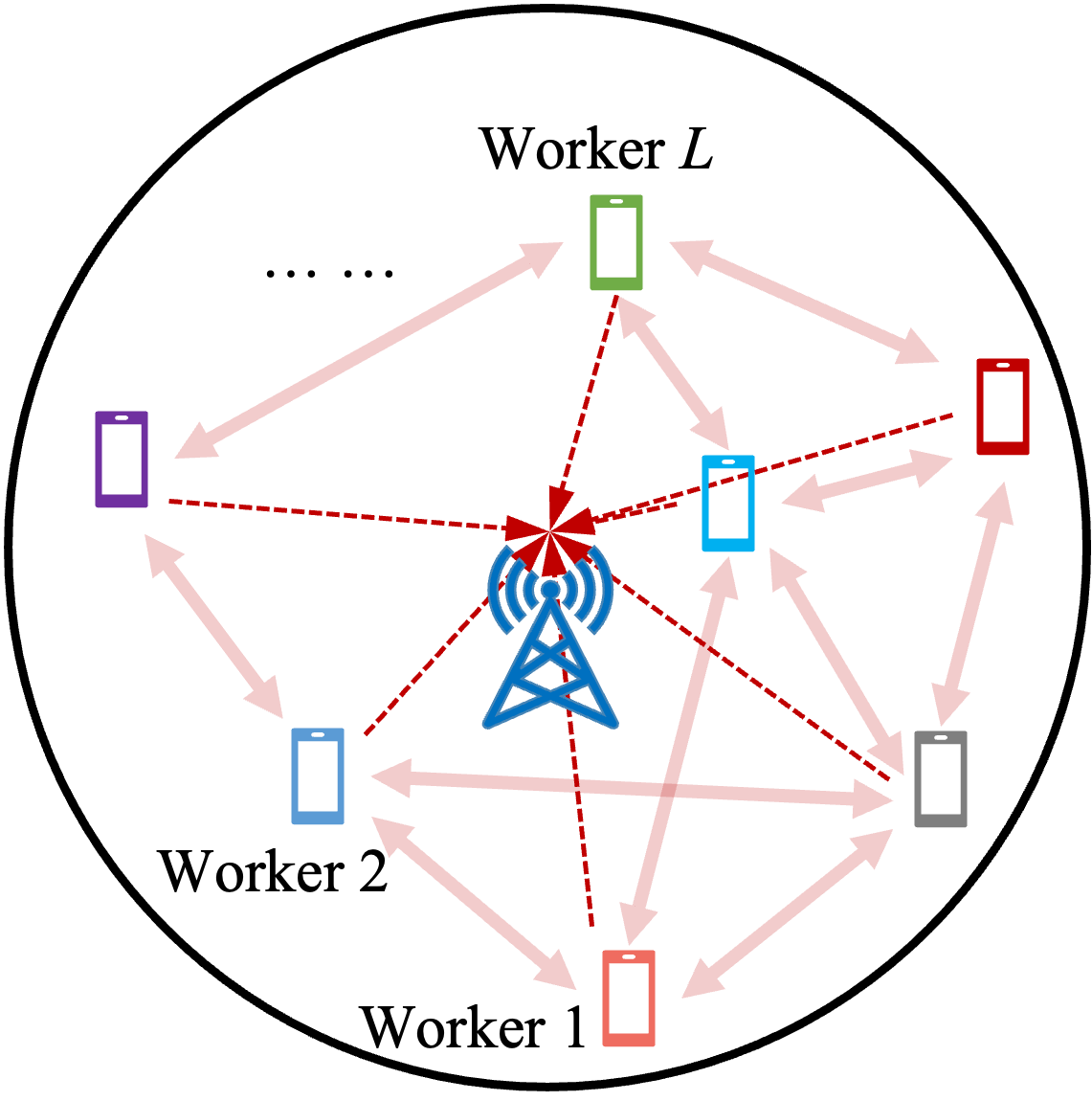}
  \captionof{figure}{
  A schematic view of the wireless system model. 
    Dashed lines indicate channels to the server. 
    Solid lines indicate interference channels between the users.
  }\label{ff:wlsys}
  \vspace*{-1em}
\end{figure}

Abstracting a signal of potentially variable length transmitted by worker $i$ as a complex scalar $s_{\text{t},i} {\,\in\,}\mbC$, the received signal at the BS is ${\bbs_\text{r}}{\,=\,}\sum^L_{i=1}{\bbh}_{i} \sqrt{p_i} s_{\text{t},i} + {\bbn_\text{c}}$, where ${\bbh}_{i} {\,\in\,}\mbC^{n_R}$ is the direct channel from $i$ to the BS, $p_i$ denotes the power that worker $i$ uses to transmit this signal, and ${\bbn_\text{c}}{\,\sim\,}\ccalN_\ccalC(\bbzero, \bbsigma^2)$ is the additive complex Gaussian noise at the receiver. 
We assume perfect channel information, i.e., ${\bbh}_{i}$ is known by both \revminor{the receiver and the transmitter.
When matched filters are applied at the receiver, the receive beamform vector for $i$ corresponds to $\bbh_i$.}
For a given vector of assigned transmit powers ${\bbp}{\,=\,}[p_1,...,p_L]$, the instantaneous signal-to-interference-plus-noise ratio (SINR)~\cite{Haenggi2009stochastic} of the link from worker $i$ to the BS is given by
\begingroup
\abovedisplayskip=2pt
\begin{equation}\label{e:sinr}
    \SINR_i{\,=\,}\frac{\alpha_i\,p_i}{1+\sum_{j\neq i}\beta_{i,j}\,p_j}, \quad\forall\,i{\,=\,}1,...,L,
\end{equation} 
\endgroup
where the channel gain and interference coefficients are computed as
\begingroup
\abovedisplayskip=2pt
\begin{equation}\label{e:csi}
\alpha_i=\frac{\|{\bbh}_{i}\|^2}{\sigma^2} \text{ and } \beta_{i,j}=\frac{|{\bbh}^H_{i} {\bbh}_{j}|^2}{\sigma^2\|{\bbh}_{i}\|^2}, \quad\forall\,i,j, \text{and }i{\,\neq\,}j.
\end{equation} 
\endgroup
We define the channel-state information (CSI) matrix  ${\bbH}{\,\in\,}\mbR^{L{\times}L}$ with $\alpha_i$ as its diagonal entries and $\beta_{i,j}$ as off-diagonal entries, i.e., $H_{i,i}{\,=\,}\alpha_i$ and $H_{i,j}{\,=\,}\beta_{i,j}$ for $i{\neq\,}j$. 
Notice that by introducing the CSI matrix, the system can now be interpreted as a directed graph with workers as its nodes and $\bbH$ as its adjacency matrix. 

The achievable data rate for worker $i$ with bandwidth $B$ is $R_i(\bbp,\bbH){\,=\,}B\log(1{+}\SINR_i)$, and its transmission time is
\begin{equation}\label{e:tau}
    \tau_i\!\left(\bbp,\bbH\right){\,=\,}\frac{Z(\bbw_i)}{R_i}, \,\,\, 
\forall\,i,
\end{equation}
where $Z(\bbw_i)$ represents the size of the transmitted local model parameters $\bbw_i$ in bits, \revminor{plus supplementary CRC bits. 
Owing to the fixed model architecture employed, $Z(\bbw_i)$ remains constant across all workers.
}

Other worker-level constants include $P_{\text{c},i}$, namely the static power consumption \revminor{by worker $i$} during the term of transmission, and the computation power consumption \revminor{coefficient $c_{\text{comp},i}$}, which is 
\revminor{the energy required for computing per bit data, depending on hardware components like the central processing unit (CPU) and microchips.
}
Hence, the total energy consumed by the local training and transmission processes is
\begin{equation}\label{e:e}
    e_{\tot,i}(\bbp,\bbH){\,=\,}c_{\text{comp},i}Z(\bbw_i) + (p_i + P_{\text{c},i}) \tau_i, \,\,\, \forall\,i.
\end{equation}
For simplicity, we proceed with uniform constants across all workers, i.e., $c_{\text{comp},i}{\,=\,}c_{\text{comp}}$ and $P_{\text{c},i}{\,=\,}P_{\text{c}}$, $\forall\,i$, and note that the non-uniform case follows exactly the same approach. 
We assume single-packet transmissions and define the packet error rate (PER)
\begingroup
\abovedisplayskip=2pt
\begin{equation}\label{e:per}
    \PER_i(\bbp,\bbH) = 1 - \exp\!\left(-\frac{m}{\SINR_i}\right), \forall\,i,
\end{equation}
\endgroup
with $m$ being a waterfall threshold~\cite{chen2020joint,xi2011general}. 
CRC is leveraged when a packet containing $\bbw_i$ arrives at the BS.
The probability of the transmission being error-free is $1{\,-\,}\PER_i$, namely the packet success rate (PSR).
We define a system-level performance metric $g(\cdot)$ as the {\it weighted sum} of the probabilities of successful transmissions
\begingroup
\abovedisplayskip=2pt
\begin{equation}\label{e:wsq}
    g\!\left(\left[\,\PSR_i\,\right]^L_{i=1}\right){\,=\,}\sum_{i=1}^L \omega_i\,\PSR_i,
\end{equation}
\endgroup
with $\omega_i$ being a scalar weight assigned to worker $i$ accounting for its data quality and/or other factors such as worker reliability or preference.
A basic and common choice under i.i.d. data distribution is to specify $\omega_i{\,=\,}k_i/K$, using the data quantity as an indicator of the data quality.  

\vspace*{-1em}
\section{Problem formulation}\label{ss:prob}
Our goal is to find an instantaneous power allocation policy $p{\,:\,}\mbR^{L\times L}{\,\rightarrow\,}\mbR^L$ to compute the transmit power magnitudes $\bbp{\,=\,}p(\bbH)$ for any $\bbH{\,\in\,}\ccalH$, where $\ccalH$ denotes a continuous distribution from which varying CSI instances are drawn.
Before each uplink transmission, the system should determine for each worker its transmit power by feeding the power policy $p$ with the real-time $\bbH$.
After a sufficient number of FL iterations, \revminor{the global model converges to an optimal point of the global loss function while satisfying the constraints imposed by our physical system. 
In cases of non-convex optimization, this point of convergence may correspond to a stationary (not necessarily optimal) point.}
Notice that only the correctly transmitted local models will be aggregated by the server each round.
{Hence, the update of the global model at the server will be a function of the channel and the power allocation.
Without being specific to any certain iteration, it can be expressed as}
\begingroup
\abovedisplayskip=2pt
\begin{equation}\label{e:agg}
\bbw(\bbp, \bbH){\,=\,}\frac{\sum_{i=1}^L k_i \bbw_i S_i(\bbp,\bbH) }{ \sum_{i=1}^L k_i S_i(\bbp,\bbH)},
\end{equation}
\endgroup
where $S_i(\bbp,\bbH)$ is an indicator function denoting whether the server actually aggregates the local model from worker $i$.
Only if the allocated $p_i\,{>}\,0$ (i.e., worker $i$ performs a transmission) and the server-received packet passes the CRC (i.e., the transmission is error-free), we deem the transmission successful and let $S_i(\bbp,\bbH){\,=\,}1$; otherwise, $S_i(\bbp,\bbH){\,=\,}0$.
\revmajor{
In the rare case where no model is correctly transmitted in a certain round, the server retains an unchanged global copy of $\bbw$ from the previous round, and the local models are then rolled back to this global $\bbw$.
}

Under the i.i.d. data assumption, the global FL model will benefit from accessing as much local information as possible.
Intuitively speaking, this is equivalent to prioritizing transmissions from workers with better links if all workers have an even number of samples, 
or prioritizing transmissions from workers with more data if all workers are in the same wireless condition. 
Hence, we adopt the weighted sum of PSR~\eqref{e:wsq} as a proxy for the quality of the global FL model.
We now formally state our main problem.

\begin{problem}\label{P:main}
Determine the optimal power allocation policy $p^\star: \mbR^{L\times L}{\,\rightarrow\,}\mbR^L$ that solves the following optimization problem
\begingroup
\abovedisplayskip=2pt
\belowdisplayskip=0pt
\begin{subequations*}
    \begin{alignat*}{3}
        p^\star{\,=\,}& \argmax_{p}\,\,  g\left( \EH\left[\PSR(\bbp,\bbH)\right] \right), \label{e:p1}\tag{P1}\\
        \text{s.t.} \quad  
        & r_{0,i} \leq \EH\left[R_i(\bbp,\bbH) \, | \, p_i > 0\right], \label{e:p1:a}\tag{\ref*{e:p1}a}\\
        &{e_{0,i} \leq \EH\left[\frac{R_i(\bbp,\bbH)}{p_i+P_{\text{c},i}} \, | \, p_i > 0\right], \forall i,} \label{e:p1:b}\tag{\ref*{e:p1}b}\\
         & \bbp = p(\bbH) {\,\in\,}[0,P_{\max}], \,\,\, \forall \bbH, \label{e:p1:c}\tag{\ref*{e:p1}c}
    \end{alignat*}
\end{subequations*}
\endgroup
\end{problem}
\noindent where $r_{0,i}$ and $e_{0,i}$ are prespecified minimum levels of desired expected data rate and energy efficiency, respectively, for worker $i$.
The expectation $\EH[\cdot]$ over the distribution $\ccalH$ (or simply $\mbE[\cdot]$ in the remainder of this paper) is applied elementwise to the variable in the bracket. 
Note that, instead of an upper constraint on energy consumption, we place a lower constraint on a variant of its inverse metric in the inequality~\eqref{e:p1:b} where the left-hand side is converted from~\eqref{e:e} such that\looseness=-1
\begin{equation*}
e_{0,i}{\,=\,}\!\left(\frac{e_{\tot_{\,0,i}}}{Z(\bbw_i)}-c_{\text{comp}}\right)^{-1}, \forall\, i{\,=\,}1,...,L,
\end{equation*}
which reflects the energy efficiency of worker $i$'s transmission.
Hence, higher values in $\bbe_{0}$ are generally preferred in an energy-constrained system, indicating more efficient use of energy during wireless transmissions.

To better understand \eqref{e:p1}, notice that we are trying to maximize $g(\cdot)$ [cf.~\eqref{e:wsq}] applied to the expected probabilities of successful transmissions given a distribution $\ccalH$ from which the CSI matrix instances are drawn. 
We seek to maximize this objective subject to three natural constraints.
In~\eqref{e:p1:a}, we require the data rate at each worker that transmits, i.e., workers where $p{\,>\,}0$, to exceed (in expectation) a prespecified rate $r_{0,i}$. 
In light of expression~\eqref{e:tau}, this requirement can be reinterpreted as demanding a transmission duration shorter than the maximum allowable time for all transmitters with non-zero allocated power.
\revminor{At the same time}, in~\eqref{e:p1:b}, the energy efficiency at each transmitting worker is required to be at least $e_{0,i}$.
Of practical importance, this constraint ensures that the total energy consumed by the wireless transmission is within a reasonable bound. 
Moreover, in~\eqref{e:p1:c}, we guarantee that the instantaneous allocated power $p(\bbH)$ for every worker does not exceed an upper bound $P_{\mathrm{max}}$.

\revmajor{On a side note, during a transmission round, the data rates may not remain constant, for early finishers may reduce interference in the network.
Nonetheless, our solution remains feasible, as the completion of transmissions by some workers can only improve the rates of the rest and does not violate any constraints.
}
It should be observed that optimally solving~\eqref{e:p1} is particularly challenging.
Apart from having a non-convex objective [cf.~\eqref{e:sinr} and~\eqref{e:per}] there are at least two additional sources of complexity:
i)~We are optimizing over the space of power allocation functions, which is inherently infinite-dimensional, and ii)~We are not given a specific channel or sets of channels but rather we have to consider expected values over a whole distribution of channels.
We explain how our proposed solution addresses these challenges in the next section.\looseness=-1

\vspace*{-1em}
\revminor{
\section{Methods}\label{s:alg}
}

Our method (termed \underline{p}rimal-\underline{d}ual \underline{g}raph convolutional power network, or PDG, see Fig.~\ref{ff:main}) parameterizes the space of power allocation functions through GCNs (Section~\ref{ss:param_gcn}) and relies on a PD restatement of the original problem (Section~\ref{ss:pd}). 
In Section~\ref{ss:opt}, we prove that~\eqref{e:p1} has a zero duality gap, and we show that the optimality of our solution solely depends on the expressive power of the parameterization.

\begin{figure*}[t]
\centering
    \includegraphics[width=17.5cm,
    trim=0 7cm 0 0cm, 
    ]{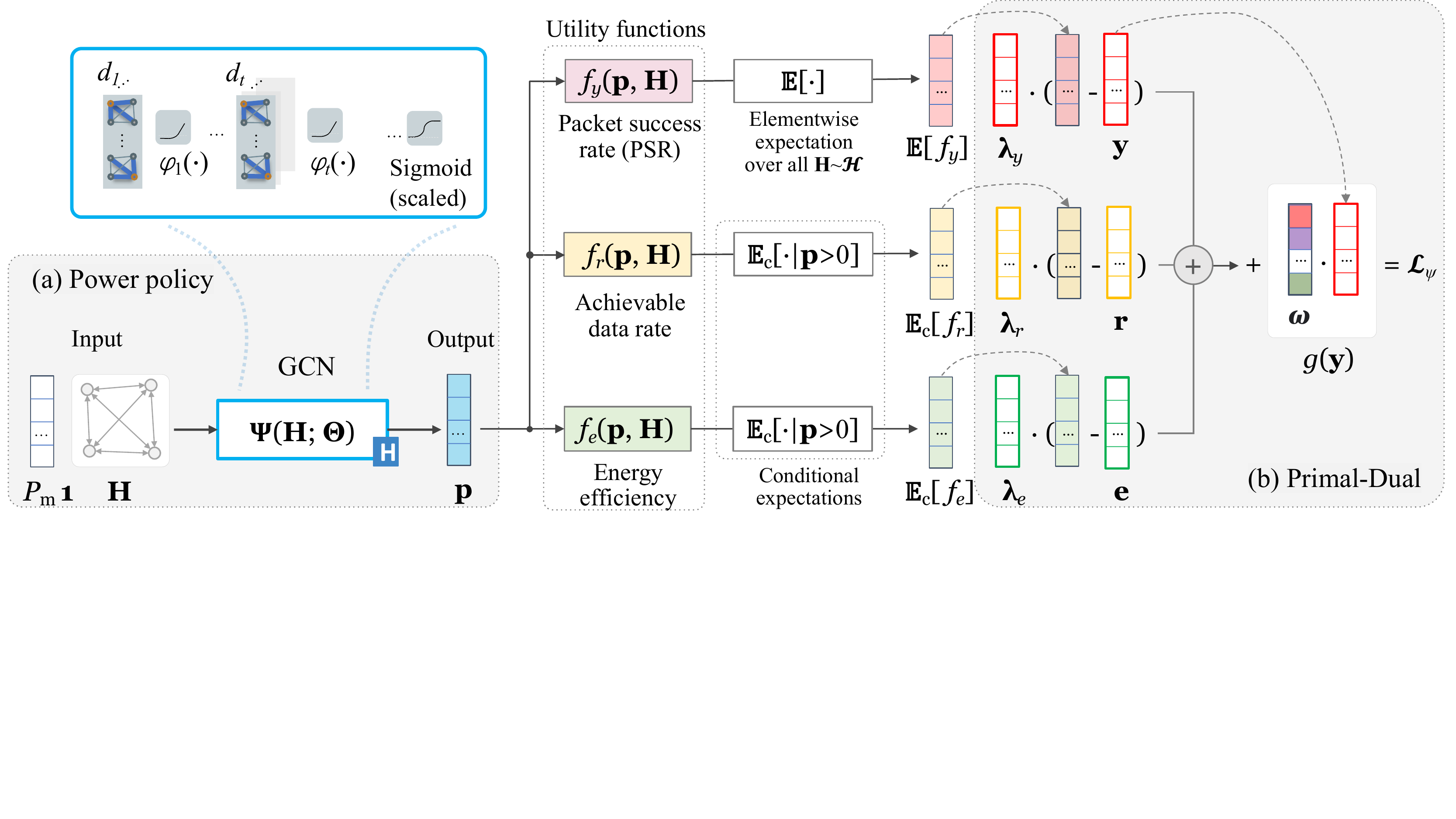}
    \caption{
    Illustration of the PDG architecture.
    It consists of a power policy block~(a) and a primal-dual (PD) learning block~(b). 
    }
    \label{ff:main}
    \vspace*{-1em}
\end{figure*}

\vspace{.3em}
\subsection{The Proposed Architecture}\label{ss:ach_pdg}
\vspace{-.3em}
Let us begin with a detailed explanation of Fig.~\ref{ff:main}. 
The PDG architecture is composed of two main components: 
(a) a power policy block denoted as (a) in Fig.~\ref{ff:main} and a PD learning block (b). 
Notationally, $d_1$ and $d_t$ represent the last dimension of the trainable parameter matrices within the corresponding GCN layers.
Between these layers, the elementwise activation function $\varphi_t$ is applied, introducing non-linearities to the network's computations.

Continuing, PDG derives three key quantities: 
the expected Partial Signal-to-Noise Ratio (PSR) values, denoted as $\mathbb{E}[f_y]$, 
the conditionally expected achievable rates, denoted as $\mathbb{E}_c[f_r]$, 
and the conditionally expected energy efficiency, denoted as $\mathbb{E}_c[f_e]$. 
These quantities are obtained through specific utility functions, which take as their inputs both the network topology $\bbH$ and the output $\mathbb{p}$ obtained at the conclusion of block~(a). 
Importantly, these expected values correspond to the right-hand-side of the constraints outlined in equations~\eqref{e:p2:a},~\eqref{e:p2:b}, and~\eqref{e:p2:c}, respectively.
Augmented by primal variables ($\bby$, $\bbr$, and $\bbe$) and dual variables ($\bblambda_y$, $\bblambda_r$, and $\bblambda_e$), PDG formulates a parameterized Lagrangian $\ccalL_\psi$, as expressed in~\eqref{e:lag}. 
Additionally, this Lagrangian incorporates $g(\bby)$, which represents a weighted sum of $\bby$, with weights determined by the predefined worker weights $\bbomega$.

The training process essentially completes the entire computational flow, progressing from block~(a) through block~(b) in a forward pass. 
Upon the completion of training, only the power policy is of interest for computing the allocated power vector $\bbp$, given any unseen $\bbH$ instances.

\vspace*{-1em}
\subsection{Parameterizing Power Policies using GCNs}\label{ss:param_gcn}

Using a deep graph neural architecture, we parameterize a subset of all possible power allocation policies $p$.
By doing so, we render the maximization in~\eqref{e:p1} tractable by maximizing over the parameters of the neural architecture instead. 
To be more precise, we define 
\begin{equation}\label{e:ppsi}
    \revmajor{p_{\psi}(\bbH){\,=\,}\Psi(\bbH;\bbTheta), }
\end{equation}
where $\Psi$ is a \mbox{$T$-layer} GCN with trainable parameters $\bbTheta$.
The input to $\Psi(\bbH;\bbTheta)$ is $\bbZ^{(0)}{\,=\,}P_{\mathrm{max}} \mathbf{1}$ and the computation in the $t^\text{th}$ layer of the GCN is given by
\begin{equation}\label{e:gcn}
	\bbZ^{(t)}{\,=\,}{\varphi_t}{\left(\hat{\bbH}\bbZ^{(t-1)}\bm{\Theta}^{(t)}\right)}, \forall\,t=1,...,T,
\end{equation}
where $\hat{\bbH}{\,=\,}\bbD^{-\frac{1}{2}}\bbH\bbD^{-\frac{1}{2}}$ is the normalized adjacency matrix with degree matrix $\bbD{\,=\,}\diag(\bbH \mathbf{1})$.
In~\eqref{e:gcn}, ${\bbTheta}^{(t)} {\,\in\,}\mathbb{R}^{d_{t-1} \times d_{t}}$ contains trainable parameters where
$d_{t}$ is the output dimension of layer $t$ with $d_0{\,=\,}d_T{\,=\,}1$, 
and $\varphi_t(\cdot)$ is an elementwise activation function. 
In Section~\ref{s:exp}, we detail the specifics of $T$, $d_t$, $\varphi_t$, and other hyperparameters used in our experiments.

In principle, one may choose any architecture as the parameterization of the power policy functional space.
Our particular choice of GCN in~\eqref{e:gcn} not only achieves good performance in practice but is also driven by the following three appealing characteristics.

\vspace{2mm}
\noindent
\textbf{Permutation equivariance.} 
GCNs are composed of graph filters, which are similar to convolutional filters in classical convolutional neural networks (CNNs)~\cite{segarra2017optimal}.
These graph filters learn to perform appropriate local aggregations through trainable parameters.
Similar to CNNs being equivariant to translations~\cite{worrall2017harmonic}, GCNs are permutation equivariant~\cite{roddenberry2022local}, meaning that the output is not affected by the ordering of nodes. 
Our parameterization entails a permutation equivariant power allocation policy, such that the allocated power is not affected by the indexing of workers, which is a key feature since this indexing is arbitrary.

\vspace{2mm}
\noindent
\textbf{Flexibility to varying number of workers.} 
In FL, the possible termination of worker participation inevitably causes variations in the communication network size.
Such size changes can occur within a short time, which is not enough to retrain the power allocation model.  
This requires the power policy to accommodate, without retraining, a number of workers $L'$ different from the $L$ workers with which it was trained.
It thus undermines the usage of any classical deep models that require fixed input and output dimensions. 
It also facilitates large-network training by iteratively training on sampled subsets.
Our PDG possesses this distinguishing capability by relying on GCNs [cf.~\eqref{e:gcn}], a critical point to be numerically showcased in Section~\ref{s:exp}.

\vspace{2mm}
\noindent
\revminor{
\textbf{Extensibility to heterogeneous systems.} 
The proposed architecture exhibits easy extensibility to systems of heterogeneous devices and traffic types. 
For instance, by modeling the network as a multigraph, we could encode different wireless technologies or protocols as different edge types.
In this case, the power policy model may be represented using relational graph convolutional networks (R-GCNs)~\cite{schlichtkrull2018modeling} to enable effective management in multigraphs. 
Moreover, cases where there are different types of workers in the network can be accommodated via heterogeneous GNNs (HGNNs)~\cite{zhao2021heterogeneous} or graph attention networks (GATs)~\cite{velickovic2017graph}, which are particularly capable of handling different types of nodes. 
}

\vspace*{-1em}
\subsection{Primal-Dual Augmented Constrained Learning}\label{ss:pd}
Having introduced the GCN parameterization of the power policy \revminor{$p_{\psi}$ in~\eqref{e:ppsi}}, we proceed to restate~\eqref{e:p1} in a manner that is amenable to a PD solution
\begingroup
\allowdisplaybreaks
\begin{subequations*}
    \begin{alignat*}{3}
        P_\psi^\star &= \max\limits_{\bm{\Theta},\bby,\bbr,\bbe}\quad  g(\bby), \label{e:p2}\tag{P2}\\
        &\text{s.t.} \quad  \bby \leq \mbE{\left[\,\PSR(\bbp_{\psi},\bbH)\,\right]}, \label{e:p2:a}\tag{\ref*{e:p2}a}\\
         & r_i \leq \mbE{\left[R_i(\bbp_{\psi},\bbH) \, | \, {p_{\psi}}_i \!> \!0\right]}, \label{e:p2:b}\tag{\ref*{e:p2}b}\\
         & {e_i \leq \mbE{\left[\,\frac{R_i(\bbp_{\psi},\bbH)}{{p_{\psi}}_i + P_{\text{c}}} \, \Big| \, {p_{\psi}}_i \!> \!0\,\right]},} \label{e:p2:c}\tag{\ref*{e:p2}c}\\
         & r_i {\,\in\,}[r_{0,i},+\infty), \,\,  \label{e:p2:d}\tag{\ref*{e:p2}d}\\
         & {e_i {\,\in\,}[e_{0,i},+\infty), \,\, \forall\,i,} \label{e:p2:e}\tag{\ref*{e:p2}e}\\
         & \revminor{\bbp_{\psi}{\,=\,}p_\psi(\bbH)}{\,\in\,}[0,P_{\max}], \,\,\, \forall\,\bbH. \label{e:p2:f}\tag{\ref*{e:p2}f}
    \end{alignat*}
\end{subequations*}
\endgroup

\revmajor{
To bring the formulation closer to
a standard PD form, we introduce~\eqref{e:p2:a} with auxiliary primal variables $\bby$, $\bbr$, and $\bbe$, which will also aid in the iterative process.
}
To treat the remaining constraints, we formulate the Lagrangian dual problem by introducing non-negative multipliers $\bblambda_y{\,\in\,}\mbR_{+}^{L}$, $\bblambda_r{\,\in\,}\mbR_{+}^{L}$, and $\bblambda_e{\,\in\,}\mbR_{+}^{L}$ associated with constraints~\eqref{e:p2:a},~\eqref{e:p2:b}, {and~\eqref{e:p2:c},} respectively. 
To simplify the notation, let $\mbE[f_y]$ denote the expectation {term and $\Ec[f_r]$ and $\Ec[f_e]$ the conditional expectation terms} in constraints~\eqref{e:p2:a},~\eqref{e:p2:b}, and~\eqref{e:p2:c}, respectively.
\revmajor{
We enforce~\eqref{e:p2:f} simply by ensuring that the image of the last non-linearity in our GCN is contained in $[0,P_{\max}]$.
}
Moreover, we denote the feasible region in~\eqref{e:p2:d} as $ \ccalR{\,:=\,}[\bbr_0,+\infty)$ and that in~\eqref{e:p2:e} as $ \ccalE{\,:=\,}[\bbe_0,+\infty)$.
The Lagrangian of~\eqref{e:p2} is then given by
\begin{equation}\label{e:lag}
    \begin{aligned}
        \ccalL_{\psi}(\bbTheta,\bby,\bbr,\bbe,&\bblambda_y,\bblambda_r,\bblambda_e)= g(\bby) + \bblambda_y^{\top} (\mbE[f_y]-\bby) \\
        &+ \bblambda_r^{\top} (\Ec[f_r]-\bbr)
        + \bblambda_e^{\top} (\Ec[f_e]-\bbe).
    \end{aligned}
\end{equation}
Following standard PD solutions of optimization problems~\cite{boyd2004convex}, this Lagrangian motivates the following iterative updating approach:
in the $n^\text{th}$ iteration, the learnable parameters are updated with step size $\gamma_{\Theta,n}$ by
\begin{subequations}
    \begin{alignat}{3}
        \bbTheta_{n+1}{\,=\,}\,\bbTheta_n + \gamma_{\Theta,n}(\nabla_{\Theta}\mbE[f_y]\bblambda_{y,n} 
        &+ \nabla_{\Theta}\Ec[f_r]\bblambda_{r,n} \nonumber\\ 
        &+ { \nabla_{\Theta}\Ec[f_e]\bblambda_{e,n} }).  \label{e:lag:1} 
    \end{alignat}    
    For the primal variables, their updates follow 
    \begingroup
    \begin{alignat}{3}        
        \bby_{n+1}{\,=\,}&\,\bby_{n} + \gamma_{y,n}(\nabla_y g(\bby) - \bblambda_{y,n}
        ), 
        \label{e:lag:2}\\
        \bbr_{n+1}{\,=\,}&\,\mathrm{proj}_{\ccalR}\left( \bbr_n + \gamma_{r,n} (-\bblambda_{r,n}) \right), \label{e:lag:3}\\
        \bbe_{n+1}{\,=\,}&\,{\mathrm{proj}_{\ccalE}\left( \bbe_n + \gamma_{e,n} (-\bblambda_{e,n}) \right),} \label{e:lag:4}
    \end{alignat}
    \endgroup
    where the terms $\gamma_{*,n}$ represent step sizes for updating the corresponding primal variables in the current iteration.
    Finally, we update the dual variables
    \begin{alignat}{3}
        \bbLambda_{n+1}{\,=\,}&\left[\bbLambda_{n} - \bbgamma_{\lambda,n}^{\top} (\bbE - \bbX_{n+1})\right]_{+}
        \label{e:lag:7}
    \end{alignat}
    with the vector $\bbgamma_{\lambda,n}{\,=\,}[\gamma_{y,n}\,\,\gamma_{r,n}\,\,\gamma_{e,n}]^\top$ denoting the step sizes of the corresponding dual variables in the $n^\text{th}$ iteration, where
    \begingroup
    \abovedisplayskip=2pt
    \begin{equation*}
        \bbLambda_n{\,=\!}\left[\begin{matrix}\bblambda_{y,n}\\ \bblambda_{r,n}\\ \bblambda_{e,n}\end{matrix}\right], \bbE{\,=\!}\left[\begin{matrix}\mbE[f_y]\\ \Ec[f_r]\\ \Ec[f_e]\end{matrix}\right], \text{and } \bbX_n{\,=\!}\left[\begin{matrix}\bby_n\\ \bbr_n\\ \bbe_n\end{matrix}\right]
    \end{equation*}
    \endgroup
    denote the matrix representation of the dual, expectation, and primal variables, respectively.
    The standard PD solution repeats iteratively until convergence. 
\end{subequations}

In our implementation, we approximate the expected values in~\eqref{e:lag:1} and~\eqref{e:lag:7} through their empirical counterparts by drawing batches of CSI instances from $\ccalH$.
This above process can be implemented to learn the parameters $\bbTheta$ that maximize the objective in~\eqref{e:p2}.

\vspace{-.6em}
\subsection{Optimality Analysis}\label{ss:opt}
\vspace{-1mm}

In this section, we first show that~\eqref{e:p1} has a zero duality gap, thus theoretically motivating the use of a PD approach as described in Section~\ref{ss:pd}.
Then, we show that the optimality gap in~\eqref{e:p2} depends linearly on the approximation capability of the parameterization of the policy function.\looseness=-1

We show the first result on a generalized version of~\eqref{e:p1}, stated as follows
\begin{subequations*}
    \begin{alignat*}{3}
        &\mathit{P}^\star = \max\limits_{p}\,\, v(p), \label{e:p1g}\tag{P1$^\prime$}
        \quad\text{s.t.}\,\,
        &u_i(p)\geq c_i, \forall\,i,\quad
        &p\in \ccalP,  
    \end{alignat*}
\end{subequations*}
with $v(p){\,=\,}g(\mbE[f_0(\bbp,\bbH)])$ for some objective function $f_0{\,:\,}\mbR^L{\times} \mbR^{L{\times}L}{\,\rightarrow\,}\mbR^L$ (e.g., PSR) and a utility function $g{\,:\,}\mbR^L{\,\rightarrow\,}\mbR$ (e.g., the weighted sum).
The constraints are described by 
$u_i(p){\,=\,}\mbE[{\fc}_i(\bbp,\bbH)\,|\,p_i{\,>\,}0]$ as the expectation of constraint function(s) ${\fc}_i{\,:\,}\mbR^L{\times}\mbR^{L{\times}L}{\,\rightarrow\,}\mbR$ (e.g., data rate or energy efficiency) conditioned on transmitting workers and $c_i{\,\geq\,}0$ as the constraint value for worker $i$.
For simplicity, we will show our result for the case of a single constraint function ${\fc}_i$ for each worker, but the same result holds for any finite number of constraint functions (e.g., two in the case of~\eqref{e:p1}).
The set $\ccalP$ describes a feasible policy function space as, e.g., the one stated in~\eqref{e:p1:c}. 
Notice that~\eqref{e:p1g} is a more general formulation than our problem of interest~\eqref{e:p1}.

To establish our results, we rely on several assumptions that are common in literature on similar topics~\cite{eisen2019learning}:

\vspace{2mm}
\noindent (AS1)\label{as:1} {The probability distribution $\ccalH$ is non-atomic.}

\vspace{2mm}
\noindent (AS2)\label{as:2} {The utility function $g(\cdot)$ is concave and non-decreasing.}

\vspace{2mm}
\noindent (AS3)\label{as:3} Denoting by $p^\star$ an optimal solution to~\eqref{e:p1g}, the set $\ccalP$ is given by
\begin{align}\label{e:calP}
    \ccalP = \{ p \, | & \, p(\bbH) {\,\in\,}[0,P_{\max}], \\
    & \mathds{1} (p_i(\bbH) = 0) = \mathds{1} (p^\star_i(\bbH) = 0) \,\,\, \forall i,\bbH \}.\nonumber
\end{align}
\noindent Most of these assumptions are naturally satisfied by our problem at hand.
In particular, \hyperref[as:1]{(AS1)} is satisfied for channel models that depend on continuous probability distributions (e.g., the Rayleigh fading channels) from which we then compute the CSI matrix via~\eqref{e:csi}. 
Regarding \hyperref[as:2]{(AS2)}, our weighted sum utility function on PSR values in~\eqref{e:wsq} evidently satisfies this assumption. 
In terms of \hyperref[as:3]{(AS3)}, the maximum power constraint $P_{\max}$ on all feasible policies is also a natural requirement.
In the second part of~\eqref{e:calP}, we require knowledge of the optimal worker selection $\mathds{1} (p^\star_i(\bbH) = 0)$.
In other words, we do not know the optimal power allocation (since that is what we are optimizing for), but we assume knowledge of whether this optimal power is either zero or non-zero for every worker $i$ and channel $\bbH$.
Admittedly, this is a stringent assumption that will be leveraged in showing our theoretical results.
Nevertheless, notice that in our experiments (Section~\ref{s:exp}), we do not assume access to this knowledge and the proposed method still performs satisfactorily in practice.

Under the introduced assumptions, the following result holds.
\begin{theorem}\label{T:1}
If problem~\eqref{e:p1g} is feasible, then, under assumptions \hyperref[as:1]{(AS1)}-\hyperref[as:3]{(AS3)}, it has a zero duality gap.\looseness=-1
\end{theorem}
\begin{proof}
Let us define the following set 
$$\ccalC=\{\bbxi\in\mbR^{L+1}\,|\,\exists\,p\in \ccalP, v(p)\geq\xi_0, u_i(p)\geq\xi_i+c_i, \forall\,i\}.$$
{The feasibility of~\eqref{e:p1g} guarantees that $\ccalC$ is not empty since $(\mathit{P}^\star, 0, ..., 0){\,\in\,}\ccalC$, where $\mathit{P}^\star$ is the maximum objective achievable by that problem.}
\begin{lemma}\label{L:1}
$\ccalC$ is a convex set. (See proof in Appendix~\ref{ap:a}.)
\end{lemma}
Let us define the Lagrangian of~\eqref{e:p1g} as
\begin{equation*}
    \ccalL(p,\bblambda) = v(p) + \sum\limits^{L}_{i=1}\lambda_i(u_i(p)-c_i), \text{for }\lambda_i\in\mbR_{+},
\end{equation*}
and hence the dual value is given by 
\begin{equation*}\label{e:dv}
    D^\star = \min\limits_{\bblambda\in\mbR_{+}^{L}}\, \max\limits_{p\in\ccalP}\, \ccalL(p,\bblambda).
\end{equation*}
In showing that $D^\star{\,=\,}\mathit{P}^\star$, we \revminor{already have} $D^\star{\,\geq\,}\mathit{P}^\star$ due to weak duality. 
What is left here is to show that $D^\star{\,\leq\,}\mathit{P}^\star$, which can be achieved by finding some $\tbblambda{\,\in\,}\mbR_{+}^{L}$ such that
\begin{equation*}
    \mathit{P}^\star \geq \max\limits_{p\in\ccalP}\,\ccalL(p,\tbblambda).
\end{equation*}
Let $\bbxi^\star{\,=\,}(\mathit{P}^\star,0,...,0)$.
Notice that $\bbxi^\star{\,\in\,}\ccalC$ but $\bbxi^\star{\,\notin\,}\ccalC^\mathrm{o}$, the interior of $\ccalC$.
The latter can be proved by contradiction:
Suppose $\bbxi^\star{\,\in\,}\ccalC^\mathrm{o}$, we could find a ball centered at $\bbxi^\star$ still in $\ccalC$, i.e., $\exists\,\epsilon{\,>\,}0$ small enough such that $\bbxi^\prime{\,=\,}\bbxi^\star{+}\epsilon$ and $\bbxi^\prime{\,\in\,}\ccalC$.
Recalling the definition of $\ccalC$, a contradiction would be raised by $v(p)\geq\xi^\prime_0{\,=\,}\mathit{P}^\star{+}\epsilon{\,>\,}\mathit{P}^\star$ violating the optimality of $\mathit{P}^\star$.

Since $\ccalC$ is convex [from Lemma~\eqref{L:1}], there must exist a vector $\bba{\,=\,}[a_0, a_1, ..., a_L]{\,\in\,}\mbR^{L+1}$ such that\looseness=-1
\begin{equation}\label{e:proof1.1}
    {\bbxi^\star}^\top\bba{\,\geq\,}\bbxi^\top\bba, \forall\,\bbxi{\,\in\,}\ccalC.
\end{equation}
Then, it must be that $\bba$ belongs to the non-negative orthant and $a_0{\,>\,}0$~\cite[Theorem~3]{paternain2022safe}. 
With this, we define $\tbblambda{\,=\,}[a_i{/}a_0]_{i{=}1}^{L}$ and\looseness=-1
\begingroup
\abovedisplayskip=-2pt
\begin{equation}\label{e:proof1.2}
    \tdp{\,=\,}\argmax\limits_{p\in\ccalP}\,\,v(p)+\sum\limits_{i=1}^L\tdlambda_i(u_i(p)-c_i).
\end{equation}
\endgroup
Let ${\tilde{\bbxi}}{\,=\,}[v(\tdp), u_1(\tdp){\,-\,}c_1,..., u_L(\tdp){\,-\,}c_L]$ and notice that ${\tilde{\bbxi}}{\,\in\,}\ccalC$.
Hence, from~\eqref{e:proof1.1} and~\eqref{e:proof1.2} we have\looseness=-1
\begingroup
\abovedisplayskip=2pt
\belowdisplayskip=-2pt
\begin{equation*}
    \max\limits_{p\in\ccalP}\,\ccalL(p,\tbblambda)=\ccalL(\tdp,\tbblambda)=[1,\tdlambda_1,...,\tdlambda_L]\cdot\tilde{\bbxi}=\frac{\bba^\top}{a_0}\cdot\tilde{\bbxi} \leq \mathit{P}^\star.\nonumber
\end{equation*}
\endgroup
\end{proof}
\vspace{-1em}

Inspired by~\cite[Theorem~1]{eisen2019learning}, one can further show that the optimality gap of the parameterized~\eqref{e:p2} depends linearly on the approximation capabilities of the parameterization of the policy function.
This suggests that if our proposed $p_{\psi}$ is a good parameterization of the power policy space, then the solution given by the PD learning iterations in~\eqref{e:lag:1}-\eqref{e:lag:7} will be close to the optimal solution of our original problem~\eqref{e:p1}.
Let us elaborate on this point by deriving the bounds of the parameterized dual value, starting with generalizing~\eqref{e:p2} as follows
\begingroup
\abovedisplayskip=-2pt
\belowdisplayskip=4pt
\begin{subequations*}
    \begin{alignat*}{3}
        P_\psi^\star &= \max\limits_{\bbTheta,\bby,\bbc}\,\, g(\bby), \label{e:p2g}\tag{\ref*{e:p2}$^{\prime}$}\\
        \text{s.t.} \quad & {\bbz_{\psi}}(\bbTheta)\geq \bby, \,\, {u_{\psi}}_i(\bbTheta) \geq c_i,\,\, \forall\,i, \\
         & \bby \geq \bbzero, \bbc{\,\in\,}[\bbc_{0},+\infty),\,\, 
         \bbTheta{\,\in\,}\ccalP_{\psi}, \,\,
         \forall\,\bbH,
    \end{alignat*}
\end{subequations*}
\endgroup
with ${\bbz_{\psi}}(\bbTheta){\,=\,}\mbE{\left[\,f_0(\bbp_{\psi},\bbH)\,\right]}$ as the elementwise expectation of the objective function, 
${u_{\psi}}_i(\bbTheta){\,=\,}\mbE{\left[{\fc}_i(\bbp_{\psi},\bbH) \,|\, {p_{\psi}}_i \!> \!0\right]}$ as the expectation of the constraint function(s) (conditioned on the workers' participation),
$c_i{\,\geq\,}c_{0,i}$ (with the feasible region $\ccalC{\,:=\,}\{c_{0,i}{\,|\,}c_{0,i}{\,\geq\,}0,\forall i\}$) as the constraint value for worker $i$,
and $\ccalP_{\psi}{\,:=\,}\{\bbTheta{\,|\,}p_{\psi}(\bbH;\bbTheta){\,\in\,}\ccalP, \forall\,\bbH\}$ as the neural parameter space. 
The Lagrangian of~\eqref{e:p2g} is then given by
\begin{equation*}\label{e:lag_g}
    \begin{aligned}
        \ccalL_{\psi}(\bbTheta,\bby,\bbc,&\bblambda_y,\bblambda_c)= g(\bby)\\
        &+ \bblambda_y^{\top} (\bbz_{\psi}(\bbTheta)-\bby)
        + \bblambda_c^{\top}(\bbu_{\psi}(\bbTheta)-\bbc),
    \end{aligned}
\end{equation*}
and hence the dual value is given by
\begingroup
\begin{equation}\label{e:dual_param}
    D_{\psi}^\star = \min\limits_{\bblambda_y,\bblambda_c{\geq}\bbzero}\,\, \max\limits_{\bbTheta{\in}\ccalP_{\psi}, \bby{\geq}\bbzero, \bbc{\geq}\bbc_0}\, \ccalL_{\psi}(\bbTheta,\bby,\bbc,\bblambda_y,\bblambda_c).
\end{equation}
\endgroup
We require two additional assumptions to bound the error introduced by the parameterization.

\vspace{2mm}
\noindent (AS4)\label{as:4}~The constraint functions ${\fc}_i(\bbp,\bbH)$ are expectation-wise Lipschitz on $\bbp$ for all $\bbH$.
More precisely, for any pair of power allocations $\bbp^1 = p^1(\bbH)$ and $\bbp^2 = p^2(\bbH)$ for $p^1, p^2 \in \ccalP$, there exists a constant $L$ such that 
\begingroup
\begin{align}
    \mbE[|{\fc}_i(\bbp^1,\bbH) & - {\fc}_i(\bbp^2,\bbH)| \,\,\big|\,\,[\bbp^1]_i{\,>\,}0 \text{ and } [\bbp^2]_i{\,>\,}0]\nonumber\\
    &\leq L \, |{\fc}_i(\bbp^1,\bbH) - {\fc}_i(\bbp^2,\bbH)|.\label{e:lipsch}
\end{align}
\endgroup
\noindent (AS5)\label{as:5}~The parameterization $p_\psi(\cdot; \bbTheta)$ is $\epsilon$-universal in $\ccalP$. 
More precisely, for any $p \in \ccalP$ there exist parameters $\bbTheta$ such that
\begingroup
\begin{equation*}\label{e:epsilon_universality}
    \mbE[\| p(\bbH) - p_\psi(\cdot; \bbTheta)\| ] \leq \epsilon.
\end{equation*}
\endgroup
\noindent Notice that \hyperref[as:4]{(AS4)} is immediately satisfied by constraint functions ${\fc}_i(\bbp,\bbH)$ that are Lipschitz continuous on $\bbp{\,\in\,}(0,P_{\max}]$, such as those considered in our problem of interest.
Furthermore, \hyperref[as:5]{(AS5)} is not as much an assumption of our problem but rather a statement about the quality of the parameterization adopted.
With this notation in place, the following result holds.

\begin{theorem}\label{T:2}
If problem~\eqref{e:p2g} is feasible then, under assumptions \hyperref[as:1]{(AS1)}-\hyperref[as:5]{(AS5)}, its dual value $D_{\psi}^{\star}$ in~\eqref{e:dual_param} is bounded by
\begingroup
\begin{equation}\label{e:bound}
    P^\star{\,-\,}\!L \epsilon\left(\|\bblambda_y^\star\|_1+\|\bblambda_c^\star\|_1\right) \leq D_{\psi}^{\star} \leq P^\star +\,L\epsilon \|\bblambda_c^\star\|_1,
\end{equation}
\endgroup
where $L$ and $\epsilon$ are the constants in \hyperref[as:4]{(AS4)} and \hyperref[as:5]{(AS5)}, respectively.
\end{theorem}
\begin{proof}
See Appendix~\ref{ap:b}.
\end{proof}

As revealed by Theorem~\ref{T:2}, the dual value of the parameterized problem~\eqref{e:p2g} becomes closer to the primal value of the \emph{non-parameterized} counterpart~\eqref{e:p1g} as $\epsilon$ decreases.
Indeed, a smaller $\epsilon$ entails a richer parametric family of policies $p_\Psi$, reducing the difference between the parameterized and non-parameterized versions of the problem.
In the limit, when $\epsilon$ approaches zero, we recover the null duality gap result in Theorem~\ref{T:1}.
It is worth noting that small values of $L$ can also lead to smaller gaps since smaller $L$ typically corresponds to better-behaved constraints [cf.~\eqref{e:lipsch}].
In practice, we nevertheless make a choice of parameterization that leads to not only a small duality gap but also a good trade-off between expressiveness and tractability.
A neural network with deep and specialized architecture can achieve both an accurate approximation of the original problem and efficient optimization of the parameterized problem.
For instance, our proposed PDG (Fig.~\ref{ff:main}) can result in good empirical performance, as we demonstrate in the next section. 
We also compare PDG to a PD solution based on multi-layer perceptrons (MLP) -- whose universal approximation capabilities are well established~\cite{hornik1989multilayer} -- and empirically show that the specialized architecture of PDG helps us outperform the graph-agnostic learning method.

\section{Numerical experiments}\label{s:exp}

We evaluate the proposed power allocation method PDG\footnote{\scriptsize{Code to implement the proposed method is available at \url{https://github.com/bl166/WirelessFL-PDG}.}} within the standard FL pipeline (Fig.~\ref{ff:fl-pip}) based on either of the two following assumptions -- i.i.d. and non-i.i.d. local data distribution.
The performance of the power allocation methods is assessed using two metrics, namely the wireless transmission performance and the prediction performance of the FL global model.
\revminor{
Before delving into results and insights, let us first introduce the data preparation, experiment setups, and candidate power policies.
}

\begin{figure}[t]
\centering
    \includegraphics[width=9cm,
    trim= 0 .5cm 8cm 1cm,
    ]{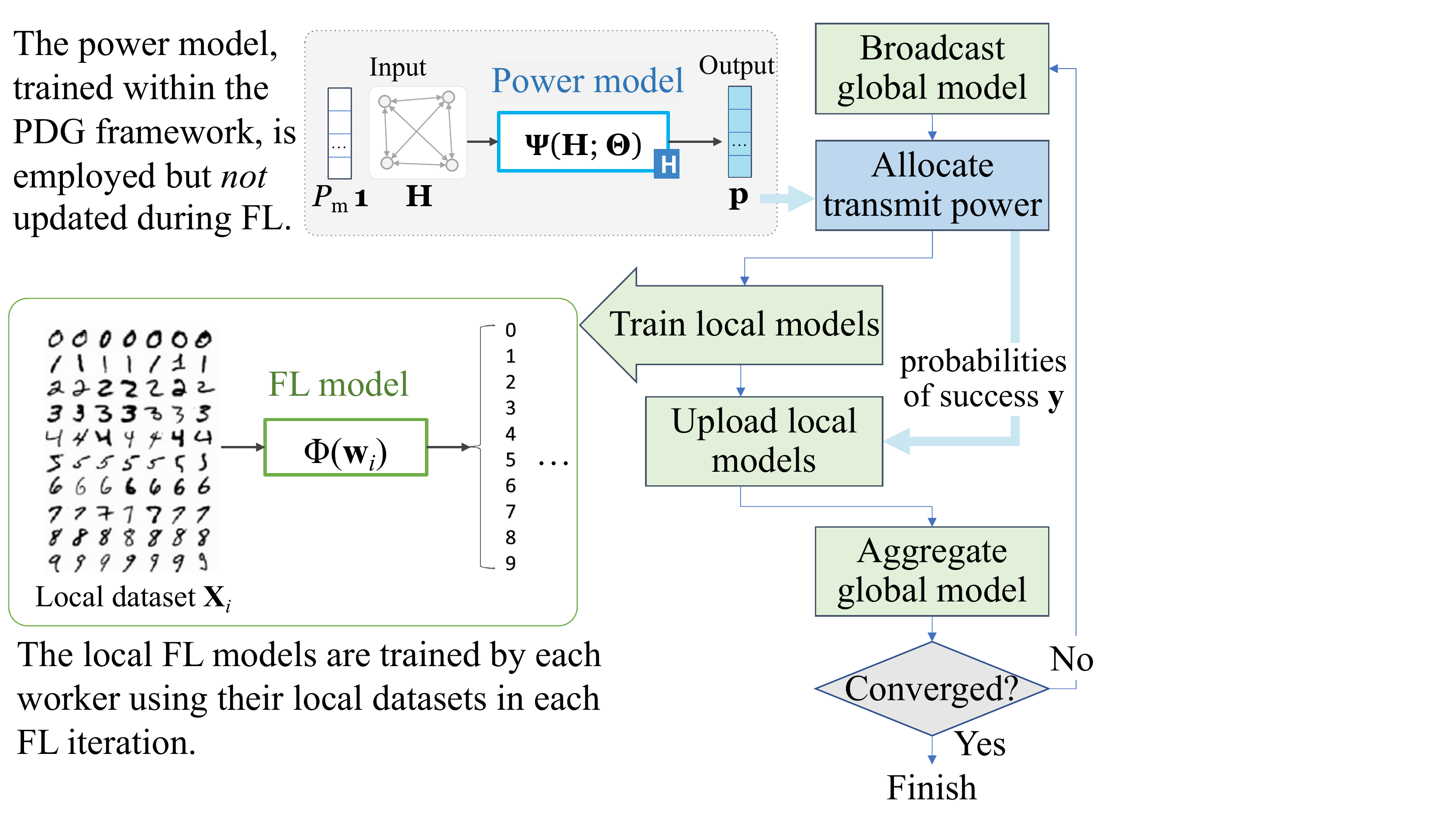}
    \caption{
    Illustration of the FL pipeline for MNIST classification.
    The proposed and candidate power allocation methods are used to allocate transmit power for local model uploads.
    }
    \label{ff:fl-pip}
    \vspace{-1em}
\end{figure}

For wireless channels, we employ the same data generator as in~\cite{matthiesen2020globally} to simulate wideband spatial (WBS) interference networks with Rayleigh fading effects\footnote{\scriptsize{\url{https://github.com/bmatthiesen/deep-EE-opt}}.} served by \revminor{one BS} with $n_R{\,=\,}10$ antennas.
The downlink communication cost is considered negligible while we focus on the uplink transmission from workers to servers.
\revminor{A number of} 3000 channel realizations are generated and equally split into training (to update model parameters), validation (to select the best-performing epoch for test), and testing (to report the performance) sets.
{Training is consistently conducted with 8-worker networks, and the default configuration for testing is also using 8-worker networks.}
For the remaining system configuration, 
{$P_{\max}$ is default at $-20 \dBW$ in all experiments except for the one that specifically targets at different $P_{\max}$ values.}
Following every intermediate layer $t{\,=\,}1,\cdots,T{-}1$, Exponential Linear Unit (ELU) is used as the non-linear activation function $\varphi_t$;
following the last layer $t{\,=\,}T$, we apply a scaled sigmoid to guarantee that the output power is contained in 
$[0,P_{\max}]$.
Learning rates $\gamma_{\Theta}{\,=\,}10^{-3}$ and $\gamma_{\{y,r,e,\lambda_y,\lambda_r,\lambda_e\}}{\,=\,}10^{-4}$ are fixed for all iterations.
A very small clamping value of $10^{-20}$ is set to stabilize the training process; 
if the predicted power is smaller than this value, it will be considered practically zero.
A total of 1000 PD learning epochs are performed in an unsupervised manner with early termination given unimproved performance in 100 consecutive epochs.\looseness=-1

\revmajor{
The plots presented in Fig.~\ref{ff:converge} illustrate the convergence of the PDG method on the validation channels. 
Fig.~\ref{ff:converge1} and~\ref{ff:converge2} display the worker-wise values pertaining to the right-hand-side of inequalities~\eqref{e:p2:b} and~\eqref{e:p2:c}, which represent delay and energy constraints, respectively.
To provide context, horizontal dashed lines denote the empirically selected constraint constants, which remain uniform across all workers. 
Successful instances where all workers satisfy both constraints are marked with blue dots on the objective curve in Fig.~\ref{ff:converge3}.
This plot also demonstrates the growth of the objective value in~\eqref{e:p2} as the learning progresses.
From these convergence plots, it becomes evident that PDG excels in learning to maximize the objective while upholding both constraints simultaneously.
Subsequently,
}
we compare PDG to one learning-based and two model-based baselines:

\begin{figure*}[t]
\centering
  \begin{subfigure}{0.26\linewidth}
      \includegraphics[
      height=4.4cm, trim= 0 0 12cm 0, clip
      ]{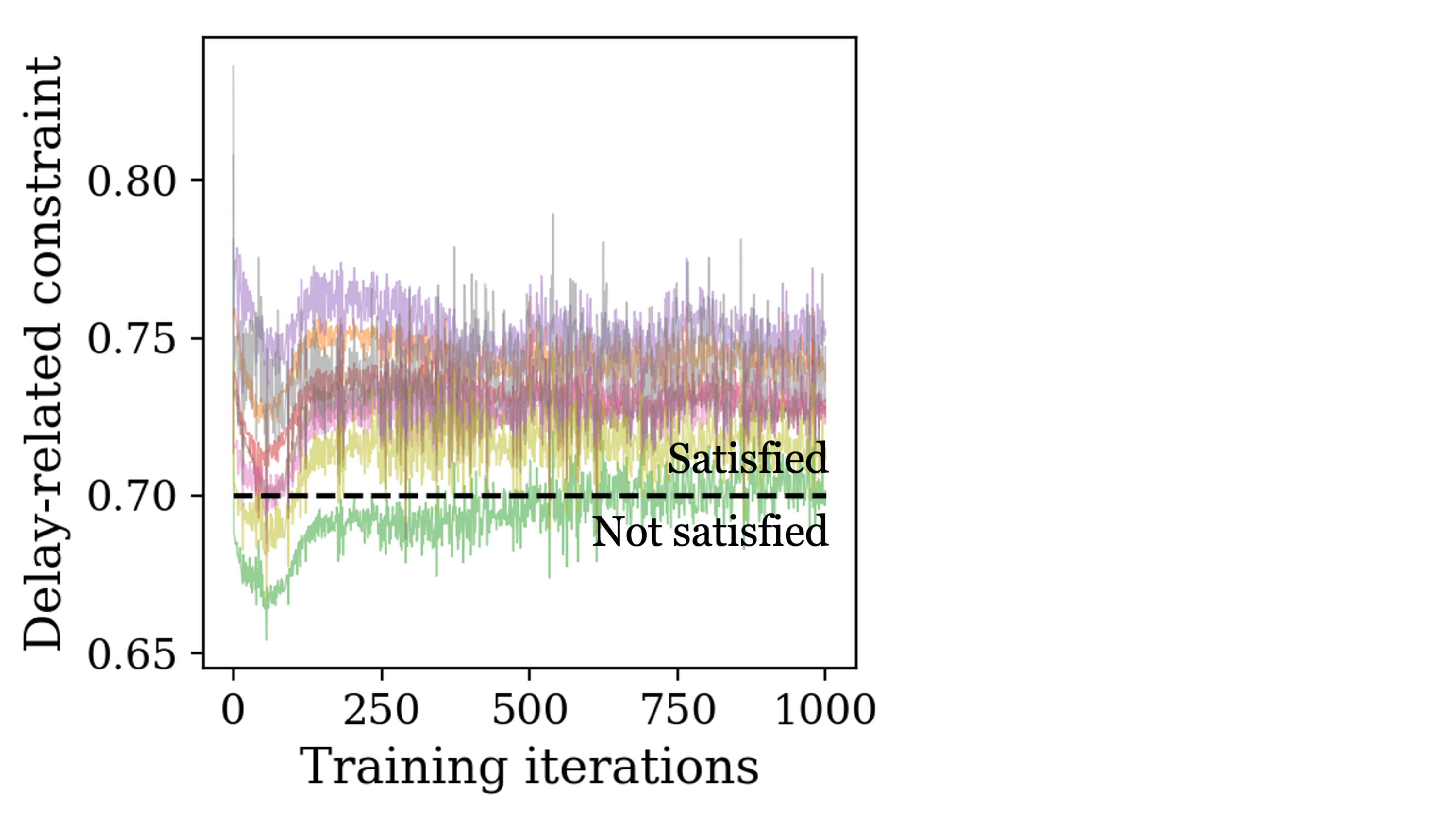}
      \vspace{-2em}
      \caption{}  \label{ff:converge1}
  \end{subfigure}\hfill
  \begin{subfigure}{0.26\linewidth} 
      \includegraphics[
      height=4.4cm, trim= 0 0 12cm 0, clip
      ]{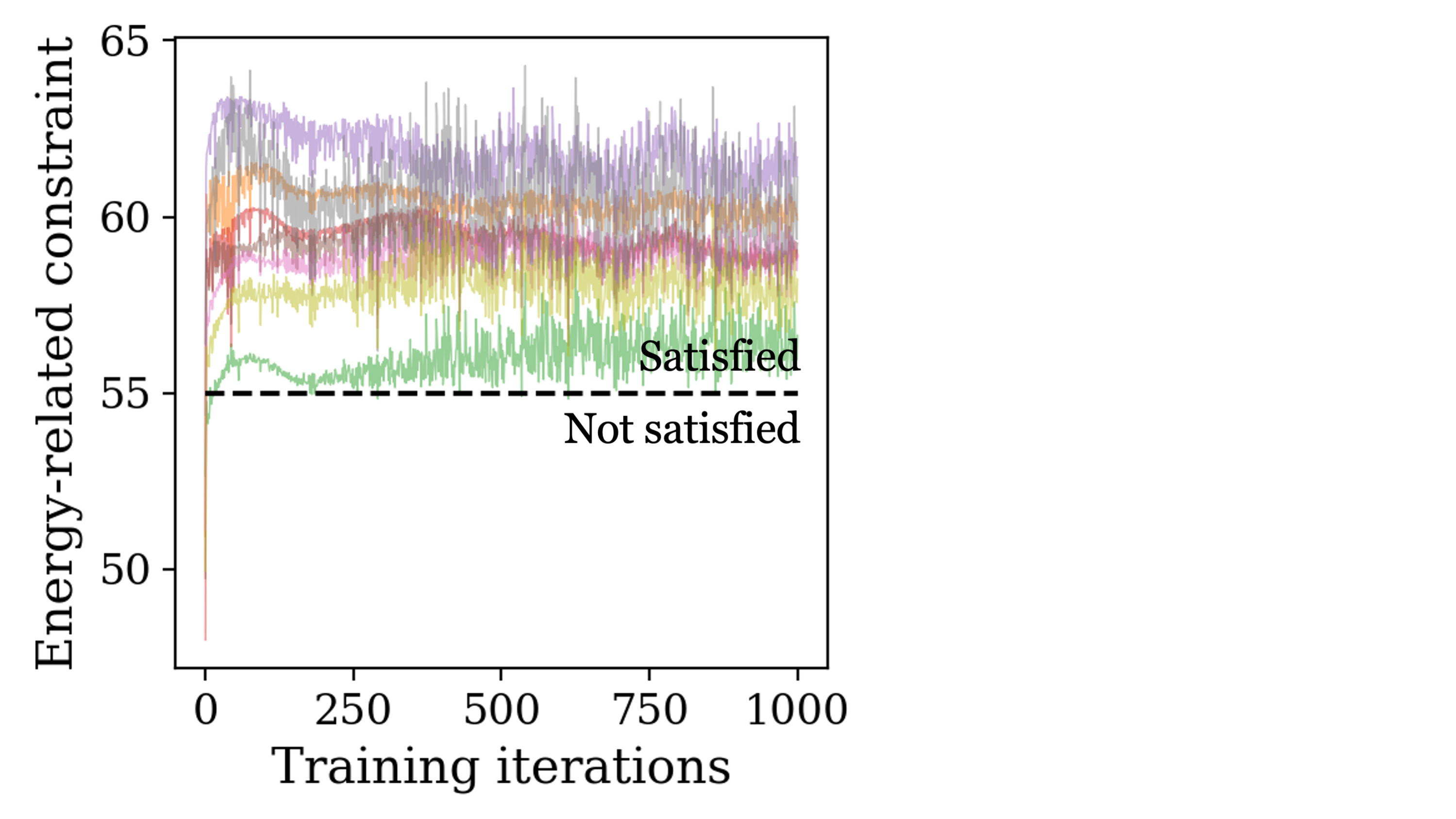}
      \vspace{-2em}
      \caption{}  \label{ff:converge2}
  \end{subfigure}\hfill
  \begin{subfigure}{0.34\linewidth} 
      \includegraphics[
      height=4.4cm, trim= 0 0 7cm 0, clip
      ]{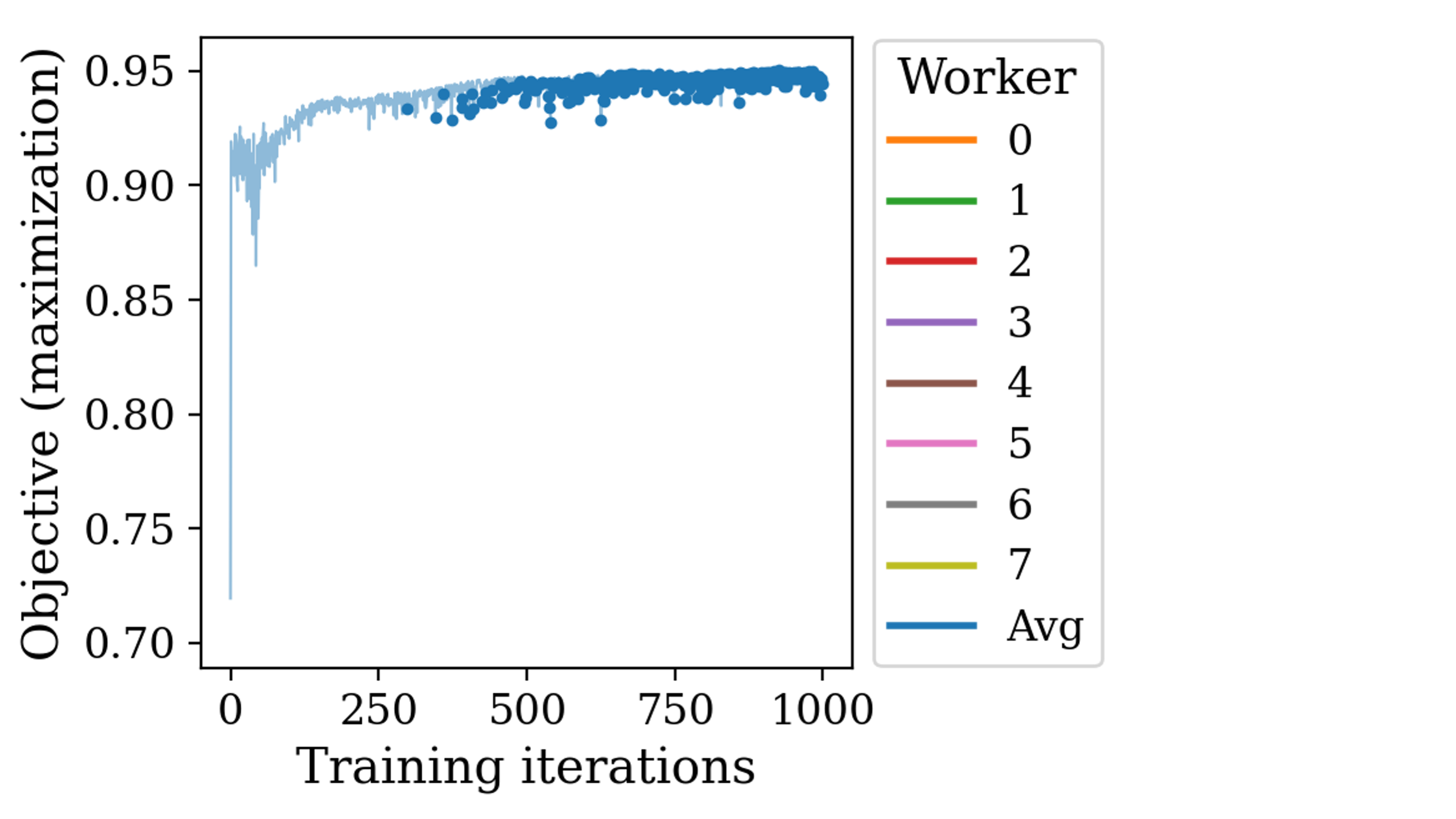}
      \vspace{-2em}
      \caption{}  \label{ff:converge3}
  \end{subfigure}
  \vspace{-.5em}
  \caption{
  \revmajor{
  Curves showing the convergence of constraints (with the constraint constants $r_0$ and $e_0$ as the respective dashed lines) and the maximization objective learned by PDG.
  The colors in the constraint plots (a) and (b) distinguish between separate workers, while the markers in the objective plot (c) highlight epochs where both constraints are satisfied by all workers. }
  }\label{ff:converge}
  \vspace{-.5em}
\end{figure*}

\begin{itemize}[topsep=0pt, wide=0pt]
    \item[1)] \textbf{Rand} is a random power policy that independently samples individual workers' transmit power from a uniform distribution $\sim\ccalU(0,P_{\max})$.
    {Occasionally, the randomly chosen power for worker $i$ violates the energy constraint.
    In these cases, worker $i$ is reassigned $p_i{\,=\,}0$.}
    \item[2)] \textbf{Orth} is the optimal power policy for orthogonal channels originally proposed in~\cite{chen2020joint}.
    It results in the participating worker $i$ transmitting at $p_i{\,=\,}\min\{P_{\max}, P_{\text{E},i}\}$, where $P_{\text{E},i}$ is the maximum power for worker $i$ that satisfies the energy constraint.
    \item[3)] \textbf{PDM} is a multi-layer perceptron (MLP) power policy model trained following the same PD algorithm as PDG. 
    Its input consists of the concatenation of the flattened $\bbH$ and $P_{\max}$.
    It is propagated through 5 hidden layers of $\{128, 256, 64, 16, 8\}$ neurons and the Leaky ReLU non-linearity in between.
    The final activation function is identical to that of the proposed architecture.
\end{itemize}

In order to guarantee that all constraints in~\eqref{e:p1} are met, the Rand and Orth allocation methods require explicit worker selection after the power is determined for all candidate workers. 
A subset of workers satisfying the constraints will participate in the current FL iteration.
In the learning-based methods PDM and PDG, however, post-allocation worker selection is not necessary. 
Instead, the models learn from data to automatically meet the constraints even for unseen instances.
To be precise, some workers may be assigned zero transmit power and thus excluded from participating. 
In this way, our experiments confirm that the constraints are satisfied on unseen channels through appropriate training.

\begin{figure*}[t] 
\centering
    \begin{subfigure}{0.27\linewidth}
     \raggedleft
     \includegraphics[height=4.8cm,
     trim=.3cm 0 0 0, 
        ]{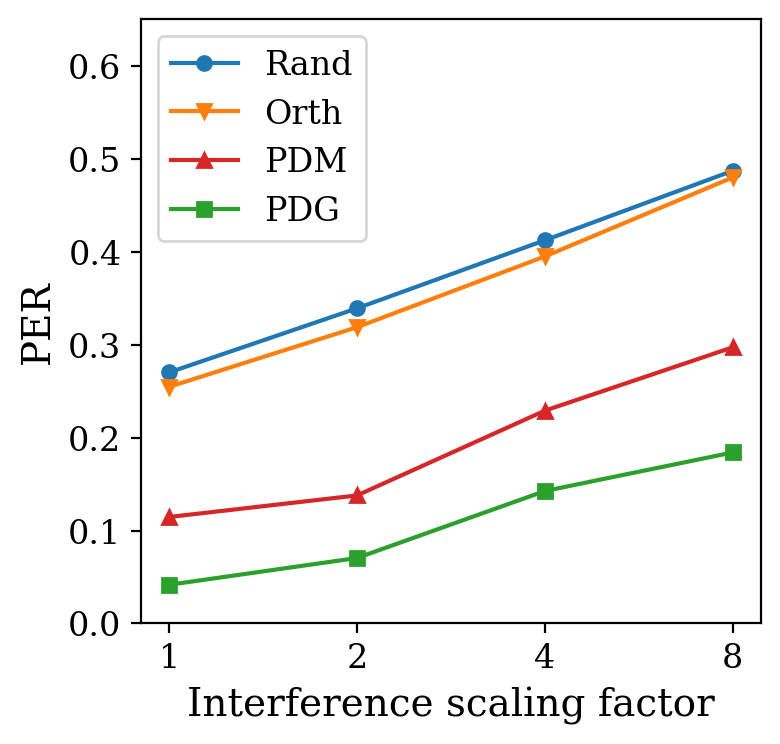}
        \vspace{-1.8em}
        \caption{}
        \label{ff:1:r1}
    \end{subfigure}
    ~\hfill
    \begin{subfigure}{0.23\linewidth} 
    \raggedright
        \includegraphics[height=4.8cm,
        trim=1.6cm 0 0 0, clip
        ]{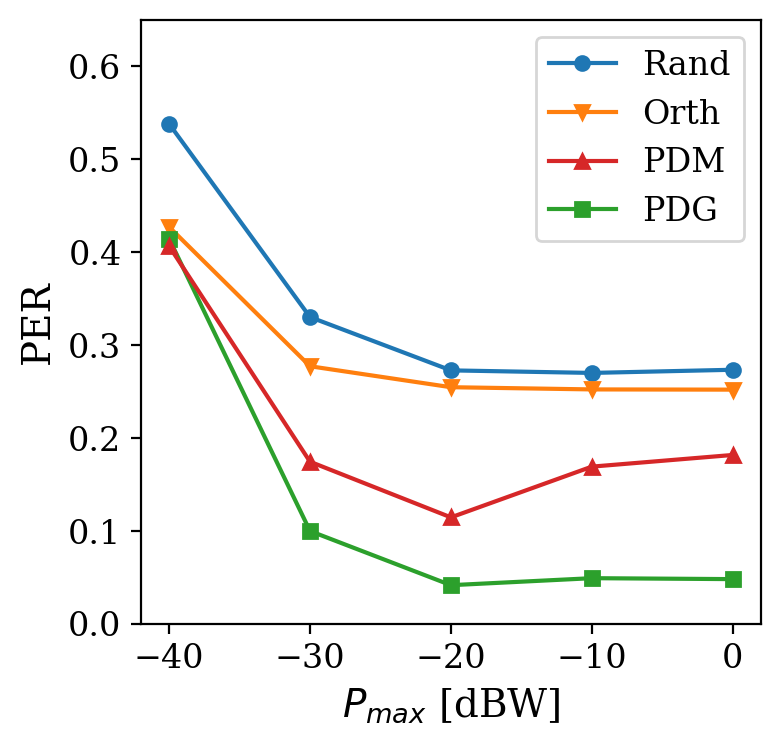}
        \vspace{-1.8em}
        \caption{}
        \label{ff:1:r2}
    \end{subfigure}
    ~\hfill
    \begin{subfigure}{0.23\linewidth} 
    \raggedright
        \includegraphics[height=4.8cm,
        trim=2.6cm 0 0 0, clip
        ]{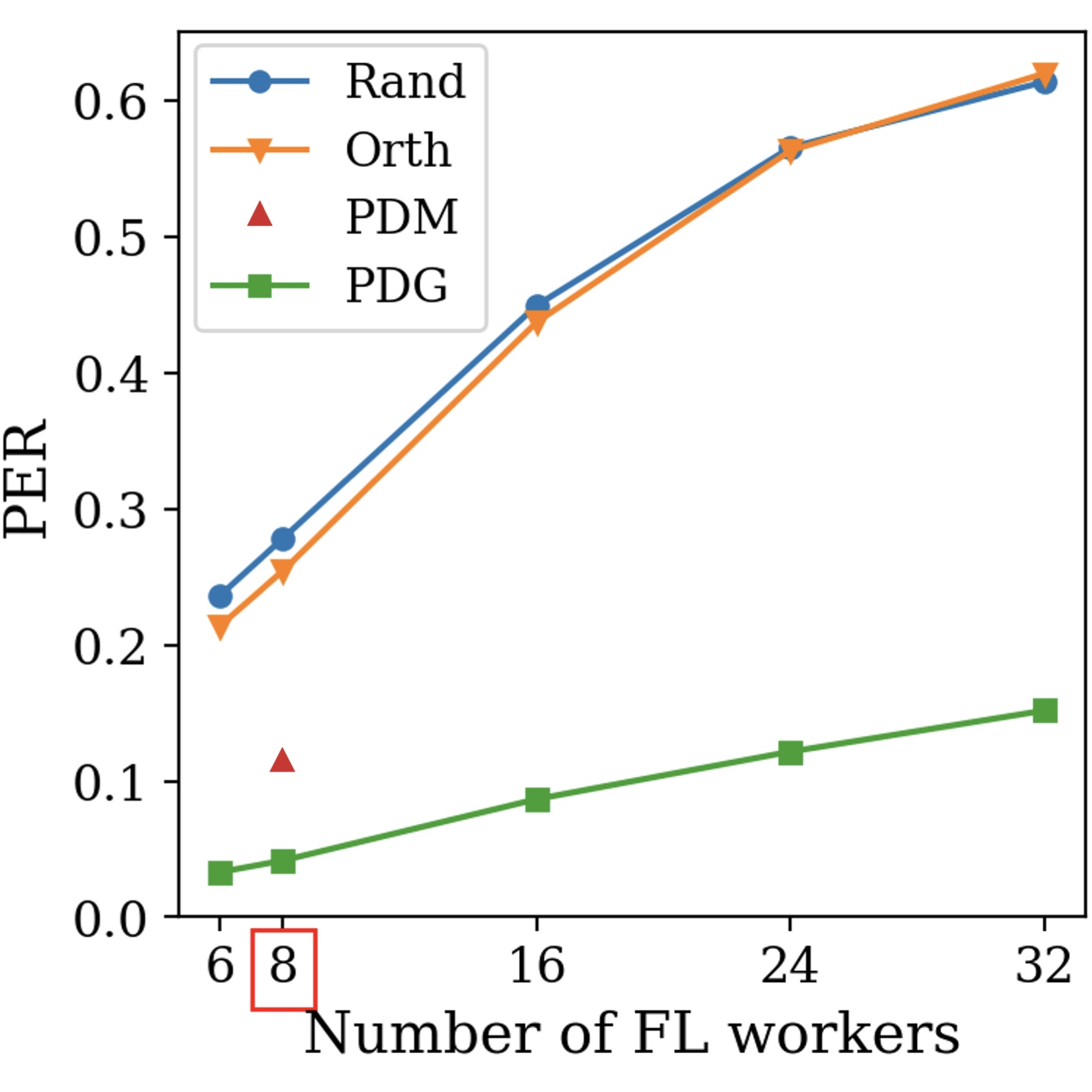}
        \vspace{-1.8em}
        \caption{}
        \label{ff:1:r3}
    \end{subfigure}
    ~\hfill
    \begin{subfigure}{0.23\linewidth} 
    \raggedright
        \includegraphics[height=4.8cm,
        trim=1.6cm 0 0 0, clip
        ]{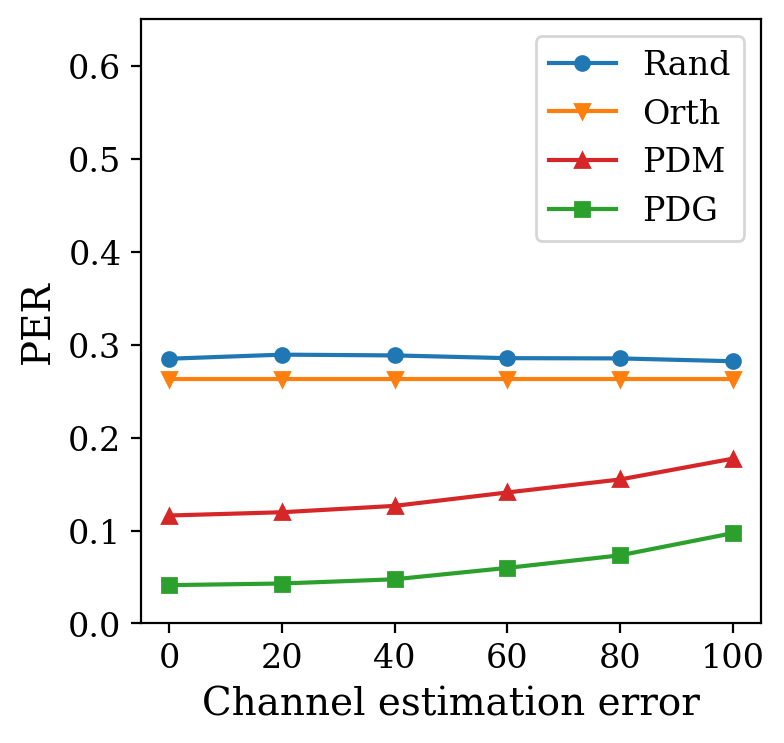}
        \vspace{-1.8em}
        \caption{}
        \label{ff:1:r4}
    \end{subfigure}
\vspace{-1.8em}
\caption{Weighted sum PER of selected workers against 
(a)~interference strength,  
(b)~maximum power constraint value,
\revminor{(c)~network size}, and
(d)~noise variance of channel estimates.
The proposed PDG outperforms the baseline power allocation strategies in all the metrics considered.
}\label{f:1}
\vspace{-1em}
\end{figure*}

\subsection{Communication Proxy Evaluation}\label{ss:commproxy}
\vspace{-.3em}

\noindent\textbf{Transmission over different interference intensities.}
In this experiment, we study the impact of the strength of wireless multiuser interference. 
To this end, we scale the interference coefficients ($\beta_{ij}$) in the CSI matrices by \{1,2,4,8\} and compare the transmission performance of candidate methods in those modified networks.
Training and the subsequent testing are conducted at each scaling factor.
Fig.~\ref{ff:1:r1} shows a communication proxy to the final FL errors, namely the weighted sum PER of all the workers that transmit.
Although for orthogonal channels, Orth is the optimal power policy, its performance within non-orthogonal channels is only comparable with that of Rand.
\revminor{In contrast}, the two PD learning methods achieve smaller PER via learning from sufficient channel samples. 
Furthermore, the graph-based PDG can transmit with even fewer errors than the topology-agnostic PDM in all tested interference intensities. 

\vspace{1mm}
\noindent\textbf{Transmission performance versus maximum power capability.}
In Fig.~\ref{ff:1:r2}, we compare the transmission performance of candidate power allocation methods against different upper bounds of transmit power in $\{-40, -30, -20, -10, 0\}\dBW$.
Training and the subsequent testing are conducted at each $P_{\max}$ level.
Intuitively, the choice of $P_{\max}$ may reflect in fundamentally different bottlenecks in power allocation. 
When $P_{\max}$ is small enough, we are essentially looking at a trivial case where the multiuser interference is small and the bottleneck is the power magnitude itself.
Consequently, it can be optimized by simply transmitting at the maximum power whenever possible.
All power policies that can either naturally follow (e.g., Orth) or pick up (e.g., PDM and PDG) this strategy have the gap between their performance closed.
On the contrary, Rand fails because it cannot enforce the maximum power strategy.
As $P_{\max}$ grows, however, interference gradually becomes the major limitation. 
Thus, we observe that the advantage of PDG's policy is more conspicuous for larger $P_{\max}$ (high interference) settings.

\vspace{1mm}
\noindent\textbf{Transmission for different number of workers.}
It is important in practice that a power policy generalizes well across different network sizes, so that one can apply a pretrained model in other wireless networks. 
In this sense, PDG is preferred over PDM, for its GCN-based power policy is versatile whereas the MLP-based power policy of the latter can only take fixed input and output dimensions. 
Also, the permutation equivariance of GCN suggests that PDG is not only applicable but also robust to changes in network size.
To show this, we investigate the transmission performance of the candidate models in different-sized FL systems and show the comparison in Fig.~\ref{ff:1:r3}.
Despite training only using an 8-worker system, PDG consistently demonstrates better wireless uplink transmission than the other two baselines in \mbox{$\{6,8,16,24,32\}$-worker} systems.
Moreover, the performance degradation of PDG due to increased multiuser interference is not as bad as that of Orth and Rand, which further underscores its generalizability.\looseness=-1

\vspace{1mm}
\noindent\textbf{Transmission based on noisy channel information.}
\revminor{ 
Channel estimation has been a subject of extensive research, leading to a variety of approaches designed to estimate wireless scenarios with differing complexities and accuracy levels~\cite{soltani2019deep,hu2018super}. 
However, it is crucial to recognize that in real-world scenarios, our assumption of perfect channel knowledge may not hold true.
Therefore,} we conduct an experiment where noisy channel estimates are taken as input to the power allocation methods.
After getting the allocations, we evaluate them using the actual channels to check their PER performance. 
Noisy channel estimates are simulated using additive white Gaussian noise (AWGN) of appropriate variances.
{As shown on the x-axis of Fig.~\ref{ff:1:r4}, different levels of variances are used to sample AWGN, which is then added to the input $\bbH$ [cf.~\eqref{e:csi}] to the allocation methods.
This is done to mimic the practically inevitable imperfection in channel knowledge. 
This plot gives an idea of} how the PER of the PD learning methods increases as a result of the deteriorating quality of channel estimation.
Both bear witness to a similar performance degradation, yet PDG consistently performs better than PDM. 
Note that it is reasonable that Rand and Orth are insensitive to noise because their proposed power schemes do not depend on specific channel realizations.

\vspace*{-1em}
\subsection{Analysis of Allocated Power and Simulated Transmissions }\label{ss:pow}

To reveal the source of performance gains that the proposed method provides, we show the histograms of allocated power of the test channel set given by PDG and PDM in Fig.~\ref{ff:pdist}.
These 8-worker channel samples are all subject to a maximum transmit power of ${-}20\dBW$ (i.e., $0{\,\leq\,}p{\,\leq\,}0.01\,\text{W}$, as shown on the horizontal axis).
It is clear how different the two power distributions are, in the sense that the PDM allocation tends to set maximum power allocation in a large number of cases (the tallest bar rightmost at $0.01\,\text{W}$) while suppressing transmissions for most of the others. 
In contrast, PDG can make finer and fairer allocations, from which one can easily observe that most of the allocated power values are smoothly distributed around $5\,\text{mW}$. 

\begin{figure}[t]
  \centering
  \includegraphics[width=8.5cm,
  trim = .5cm 0 0 0 ]{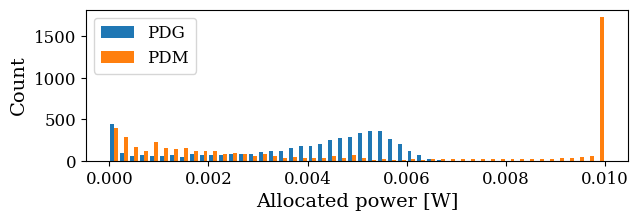}
  \vspace{-.5em}
  \captionof{figure}{
  Histograms of allocated transmit power values given by the proposed allocation method PDG and its topology-agnostic counterpart PDM. 
  }  \label{ff:pdist}
  \vspace*{-1em}
\end{figure}

Next, we proceed to perform the FL training process using real-world datasets. 
At the beginning of each FL global iteration, a channel instance is independently sampled from a WBS channel distribution. 
With this dynamic setting of the wireless environment, the FL process is simulated on a (by default) 8-worker system for 50 global iterations. 
The uplink transmissions at each iteration are simulated based on the computed PER values (which are used as probabilities with which transmission may fail) of individual workers. 
We display a box plot of the average number of successful transmissions made on one attempt (Fig.~\ref{ff:fl-suc}).
Since the system has 8 workers, this number will be 8 in the ideal case {(marked as `Ideal FL' in the figure)} where everyone gets to transmit with no loss, as shown in the figure.
We plot the boxes for each candidate allocation method, and PDG performance is observed to have not only the largest average but also the smallest variance. 
The methods of PDM, Orth, and Rand have worse performance. 
Thus, it is reasonable to expect that the PDG will result in the best FL performance (in the i.i.d. data case, at least), which we verify in the next section.

\begin{figure}[t]
  \centering
  \includegraphics[width=8cm,height=4.4cm,
    trim=.5cm 0 .5cm 0]{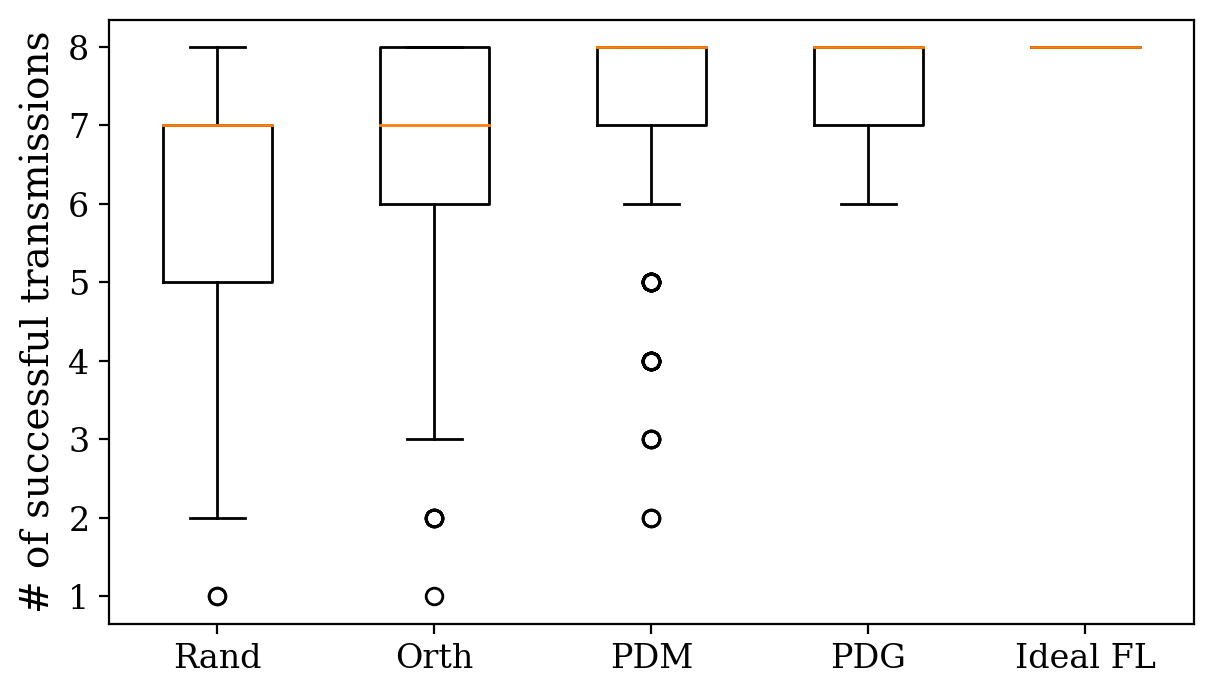}
\vspace{-.5em}
  \captionof{figure}{
  Comparison of power allocation methods on the average number of successful transmissions made in one FL global iteration during a 50-iteration simulation process. 
  }\label{ff:fl-suc}
\vspace*{-1em}
\end{figure}

\vspace*{-1em}
\subsection{Learning Performance of i.i.d. Federated Tasks}\label{ss:fed}

The power allocation methods in this paper serve the ultimate goal of promoting the convergence speed and learning accuracy of FL tasks.
\revmajor{In this regard, we test their performance on three FL tasks of different types: regression, binary text classification, and multi-class image classification. }

\vspace{2mm}
\noindent\textbf{\revminor{Air quality regression.}}
We consider air quality prediction using the UCI Air Quality dataset~\cite{de2008field}, which includes 8-dimensional hourly sensor features and hourly averaged $\text{O}_3$ levels as ground truth. 
The number of data samples at each worker is drawn from a uniform distribution $\ccalU(20,200)$.
This models the particularly challenging scenario where local workers only have access to a limited amount of data.

The regression learning model is a single-layer (50 hidden nodes) feedforward neural network with hyperbolic tangent non-linearity and mean-square error loss. 
Recalling the definition of the objective $g$ in~\eqref{e:wsq}, we determine the worker-specific weights $\bbw$ in $g$ based on the size of each worker's local dataset resulting from the uniform sampling process just introduced.
During local training, we use a batch size of 8 and the Adam optimizer with an empirical learning rate of $8{\times}10^{-4}$ for 50 FL iterations.\looseness=-1

A dynamic wireless environment is simulated by generating a new channel realization from the same WBS distribution at the beginning of every FL iteration.
\revminor{Then}, pretrained (or preset) power allocation methods are applied to determine the transmit power at each worker for uploading their local models to the server in the current FL iteration. 
Also based on the real-time channels, the PER of each transmission is computed, which provides the probabilities \revminor{from which we draw} the transmission outcomes (successful or not).
Should a transmission be deemed successful, the corresponding local model will be integrated into the global [cf.~\eqref{e:agg}] in the current FL iteration.\looseness=-1

In terms of the local training, owing to the limited number of local samples, participating workers perform a full epoch with their own on-device data during each FL iteration. 
To provide an empirical upper bound of FL performance, we consider the ideal FL scenario where transmission is lossless (i.e., all workers transmit with $100\%$ success rate in all iterations).
We report the average performance of the scaled root-mean-square errors (RMSE) against the number of FL iterations on a 100-sample test set over 10 random runs, as shown in Fig.~\ref{ff:fl-uci}. 

\begin{figure}[t]
\centering
    \includegraphics[width=8cm,
    trim=0.4cm .5cm 0.4cm 0]{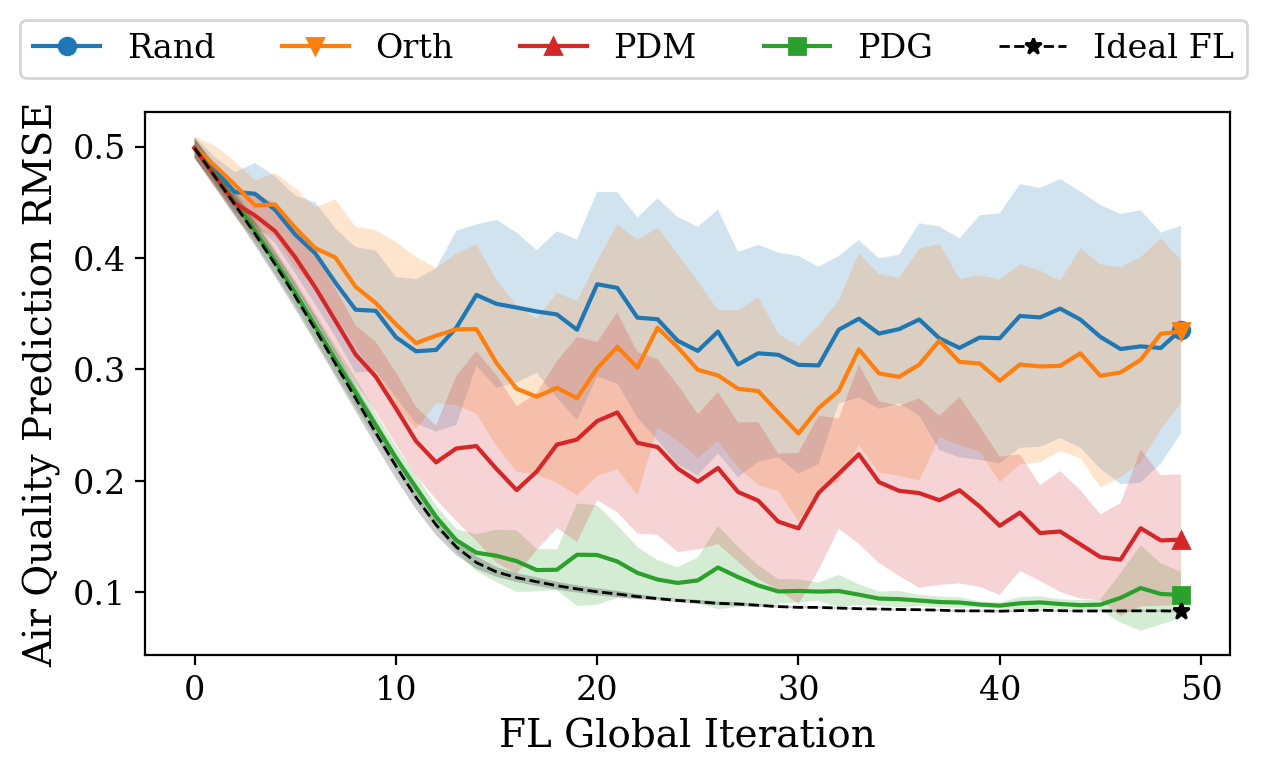}
    \caption{\revminor{Federated air quality regression, evaluated by the global model's scaled RMSE} at the end of each FL iteration.
    \revminor{Lower values reflect better power allocation.}}
    \label{ff:fl-uci}
    \vspace{-1em}
\end{figure}

\revmajor{
\vspace{2mm}
\noindent\textbf{Text sentiment classification.}\label{exp:text-iid}
We leverage the Internet Movie Database (IMDb) reviews dataset~\cite{lecun1998gradient} for binary sentiment classification.
This task serves as a benchmark for understanding sentiment in text data.
The IMDb dataset contains 50K data samples of balanced positive and negative sentiment labels. 
We reserve a subset of 7500 samples for test and randomly distribute another subset of 17.5K samples among 8 workers, ensuring that the ratio $k_i{/}K$ is identical to that of the previous task for all $i$.

We employ 50-dimensional pre-trained word embeddings~\cite{pennington2014glove} in conjunction with a bidirectional long short-term memory network (BiLSTM)~\cite{schuster1997bidirectional}, comprising 2 layers, each with 128 hidden nodes.
Due to the increased dataset size, we increase the batch size to 100 for the sake of efficiency.
Binary cross-entropy loss and the Adam optimizer with an empirical learning rate of $0.002$ are utilized.

The wireless setups are similar to the regression case, including the transmission-loss-free reference of ideal FL. 
Each time global aggregation occurs, we assess the performance of the global model by computing its classification error rate on the test set.
Based on various power allocation methods, the results are shown in Fig.~\ref{ff:fl-imdb}.
Similar to the first case, we report their average performance over 10 random realizations. 
}

\begin{figure}[t]
\centering
    \includegraphics[width=8cm,
    trim=0.3cm .5cm 0.5cm 0, 
    ]{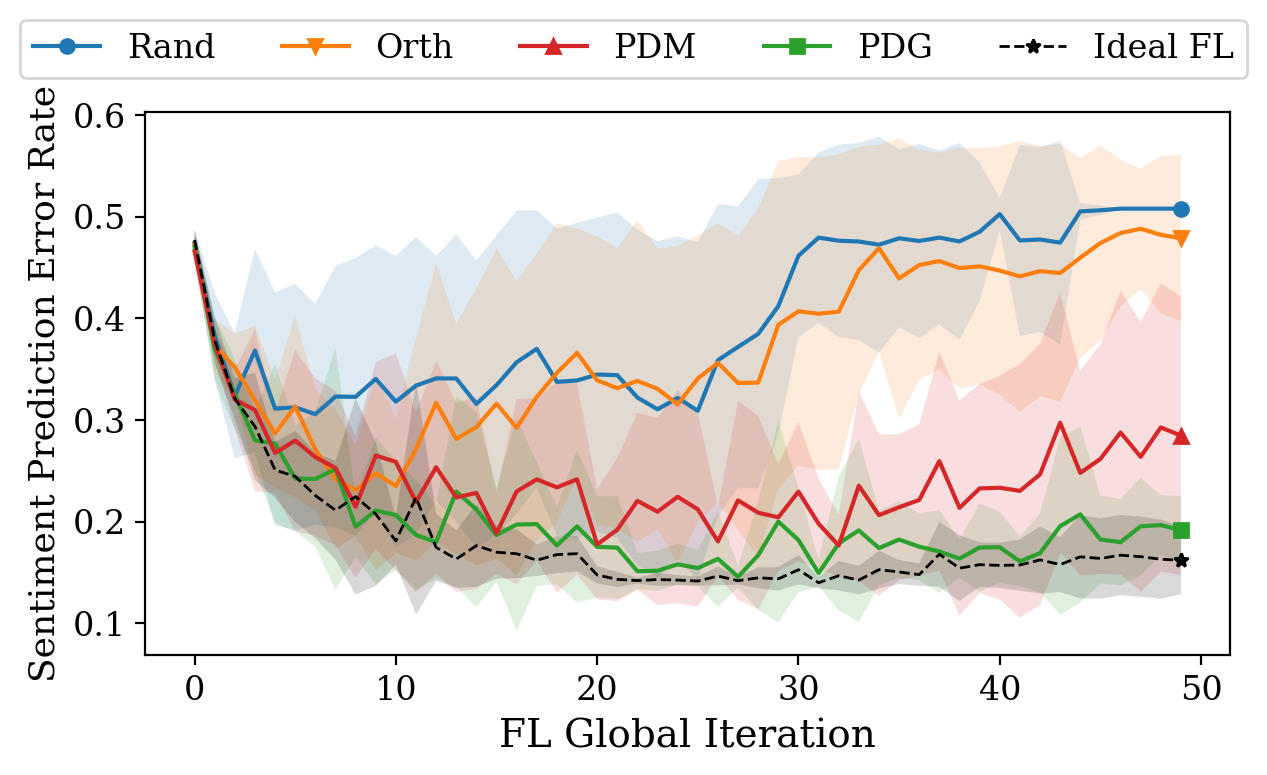}
    \caption{
    \revmajor{Federated IMDb sentiment classification error rates of the global FL model as the number of global iterations grows.
    Lower values reflect better power allocation.}
    }
    \label{ff:fl-imdb}
\vspace{-1em}
\end{figure}

\vspace{2mm}
\noindent\textbf{\revminor{Image classification.}}\label{exp:cls-iid}
We use the MNIST dataset~\cite{lecun1998gradient} for multiclass classification, which consists of $28{\times}28$ binary images (later flattened as \revminor{$784$-element} vectors) representing handwritten digits from 0 through 9. 
This task follows the same essential setups as the regression case, \revmajor{also including the number of hidden nodes and non-linearity in the FL model. }
In the meantime, we make some modifications to better suit the multi-class classification task. 
The input dimension is increased to 784, and the output dimension is set to 10 to accommodate the classification of 10 classes. 
Additionally, we use the cross-entropy loss, a learning rate of $0.001$, and 16 samples per mini-batch for local model training based on experimentation.

More notably, having a large number of MNIST samples available, we are able to increase the size of local datasets by $25{\times}$.
It therefore enables each worker to hold between 500 to 5000 data samples. 
Subsequently, instead of performing a full local epoch at each FL iteration, we utilize a fixed number of 5 local iterations that may not necessarily transverse all local samples. 
A total of 100 FL iterations are performed.
This strategy facilitates synchronization and improves efficiency by minimizing wait time for workers with a large number of samples to complete training.
These modifications are necessary to ensure that the FL model can adequately handle the classification task and produce reliable results.

Similar to the regression case, we report the average performance \revminor{in Fig.~\ref{ff:fl-mnist}}, measured by the prediction error rate, on a left-out test set of 1000 samples over 10 random realizations. 

\begin{figure}[t]
\centering
    \includegraphics[width=8cm,
    trim=0.5cm .5cm 0.5cm 0, 
    ]{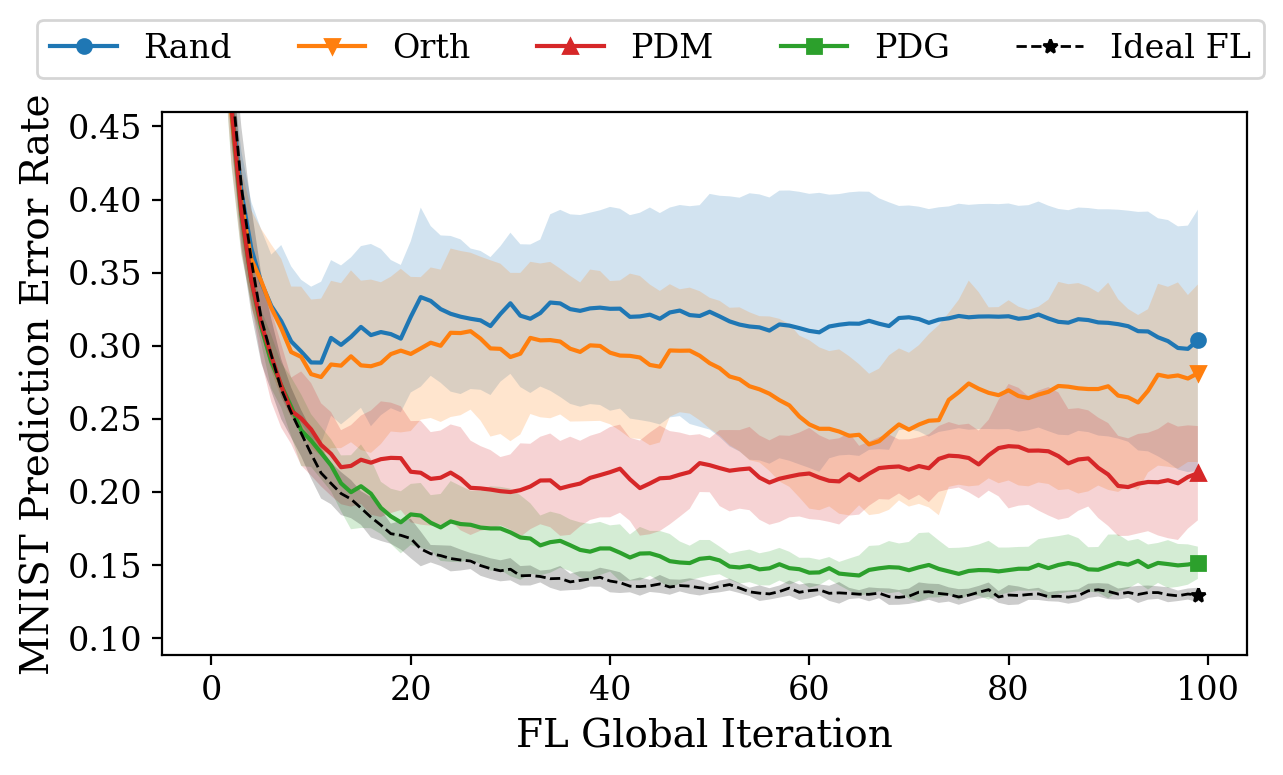}
    \caption{Federated MNIST classification based on different power methods. 
    Error rate of the FL global model at the end of each FL iteration.
    \revminor{Lower values reflect better power allocation.}
    }
    \label{ff:fl-mnist}
\vspace*{-1em}
\end{figure}

\revmajor{
The key takeaways are consistent across Figs.~\ref{ff:fl-uci},~\ref{ff:fl-imdb}, and~\ref{ff:fl-mnist}. 
Tested with various data and learning models, FL with PDG consistently outperforms all other competitors.
It is showcased that PDG accelerates the convergence speed of wireless FL, stabilizes the learning processes, and narrows the performance gap toward the ideal FL scenario.
In particular, Fig.~\ref{ff:fl-imdb} provides clear examples of federated convergence failure when the transmit power is allocated using naive policies, namely Rand and Orth.
This can be attributed to the large size of LSTM parameters and the complexity of the text task, underlining once again the significance of selecting an appropriate power allocation method for challenging FL tasks.
}

\vspace*{-1em}
\subsection{Non-i.i.d. Data Distribution}\label{ss:het}

Non-i.i.d. assumptions are reasonable in cases where the FL workers have access to different classes of labels and quality of data.
In this section, we employ the MNIST dataset once again to show two typical non-i.i.d. cases where 
i)~Local datasets contain additive white Gaussian noise of different levels (Fig.~\ref{ff:gaus-mnist}), and 
ii)~The label distribution in each worker is skewed (Fig.~\ref{ff:dir-mnist}).

In the first non-i.i.d. scenario, local \emph{feature} distributions are skewed, i.e., $\bbn_i{\,\sim\,}\ccalN_{k_i\times w\times h}(0,\eta_i^2)$, where $\bbn_i$ is the noise added to the $k_i$ images of size $w{\,\times\,}h$ at worker $i$, and $\eta_i^2$ is the worker-specific noise variance. 
{
Intuitively, the less noisy the data of a worker, the more important it should be.
Hence, we define an empirical mapping from noise levels to importance scores as $\omega_i{\,=\,}\exp[2-\frac{\eta_i-\min(\bbeta)}{30}],\forall i$.
Fig.~\ref{ff:gaus-mnist} displays the worker importance scores $\bbomega$ as will appear in~\eqref{e:wsq} and the corresponding noisy versions of a digit.} 
This is for illustrative purposes only, since the actual dataset contains different images for each worker, and therefore they would not have different noisy versions of the same digit.

In the second non-i.i.d. scenario, local \emph{label} distributions are skewed. 
In other words, FL workers are prone to imbalanced data labels on top of uneven local data quantities.
We simulate this imbalance using a Dirichlet distribution, which is a common approach to create a controllable level of label skewness by partitioning a balanced dataset into subsets containing imbalanced labels~\cite{yurochkin2019bayesian,wang2020tackling}.
More precisely, we sample $\bbtheta_d{\,\sim\,}\text{Dir}_L(\kappa)$, with $\kappa$ being a concentration parameter, and assign a $\theta_{d,i}$ proportion of class-$d$ data to worker $i$.
Tuning $\kappa$ smaller results in more imbalanced distributions. 
We tested several choices of $\kappa$ and empirically fix on $\kappa{\,=\,}1$.
Having partitioned the entire dataset, we randomly sample a desired number $k_i$ of data instances as the local dataset for each worker $i$.
The consequent local label histograms can be found in Fig.~\ref{ff:dir-mnist}.

\begin{figure}[t]
\begin{subfigure}{\linewidth}
\centering
    \includegraphics[width=8.2cm,height=2.4cm,
    trim=1cm 10cm 3cm 0]{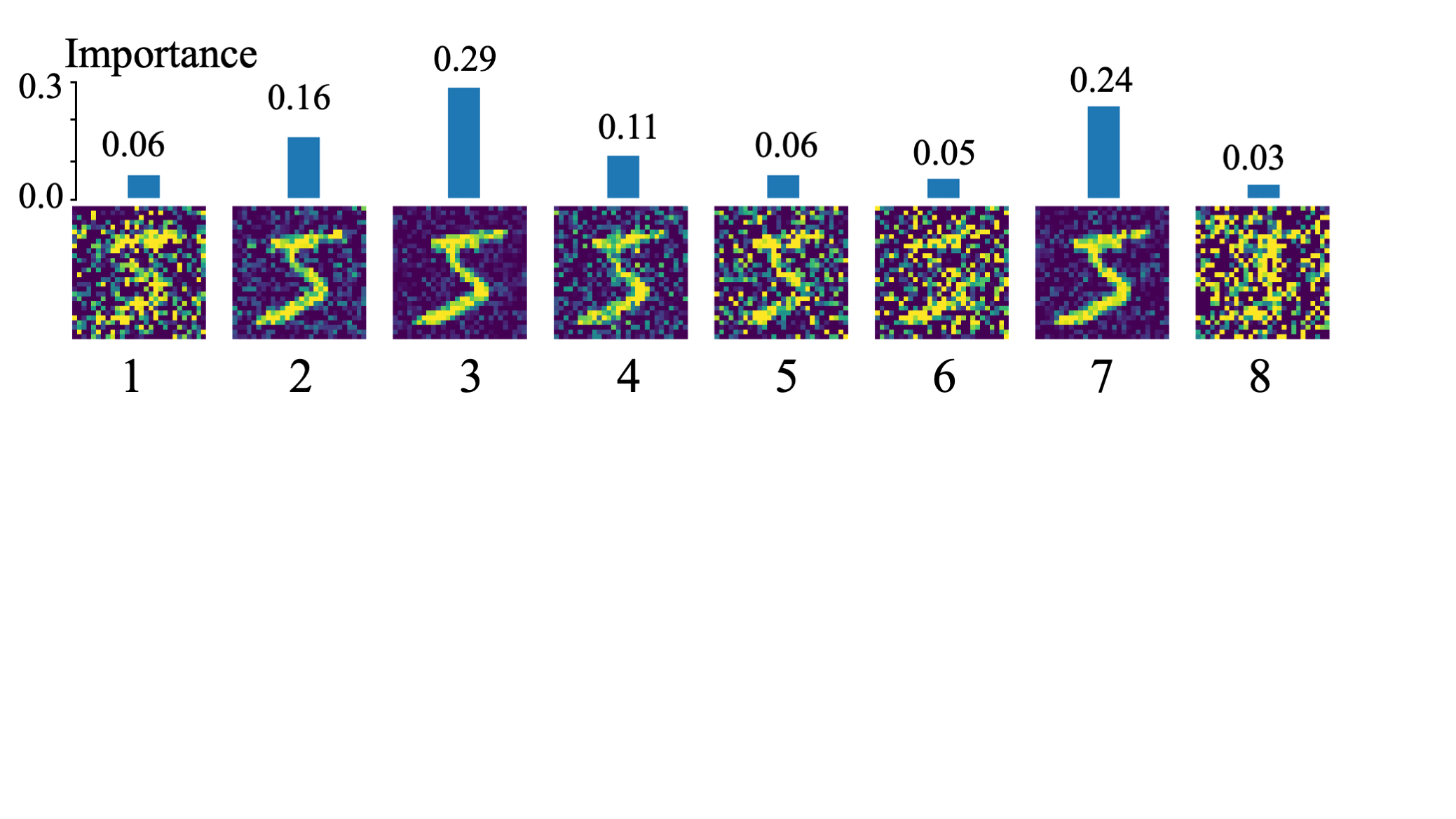}
    \caption{Heterogeneous additive Gaussian noise.
    The bars denote corresponding worker weights that reflect local data quality. 
    }
    \vspace{1em}
    \label{ff:gaus-mnist}    
\end{subfigure}\\
\begin{subfigure}{\linewidth}
\centering
    \includegraphics[width=8.4cm,height=3.1cm,
    trim=0 .5cm 0 .5cm]{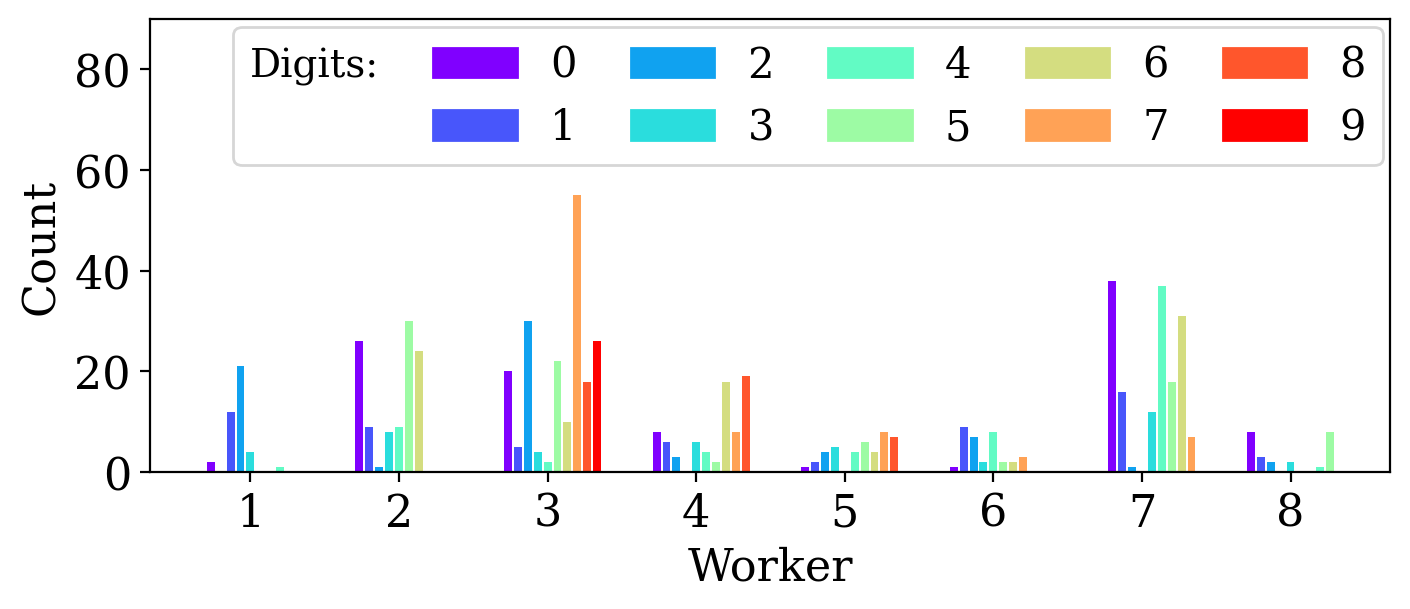}
    \caption{Skewed label histograms drawn from a Dirichlet distribution.
    }
    \label{ff:dir-mnist}    
\end{subfigure}
\vspace{-1.5em}
\caption{
Heterogeneous local data distribution types.
}
\label{ff:noisy-mnist}    
\vspace{-.5em}
\end{figure}

Fig.~\ref{ff:non-iid-mnist} summarizes the results of non-i.i.d. wireless FL experiments in which the power schemes are determined by the candidate and proposed allocation approaches. 
It can be observed that the FL models resulting from PDG always produce more stable and accurate performance than competitors in learning to classify handwritten digits.
These results suggest that PDG can serve as a promising power solution for enabling efficient and reliable wireless FL in real-world scenarios.

\begin{figure}[t]
\centering
    \includegraphics[width=7.6cm,
    trim=0.1cm .5cm 1.2cm 0cm, clip
    ]{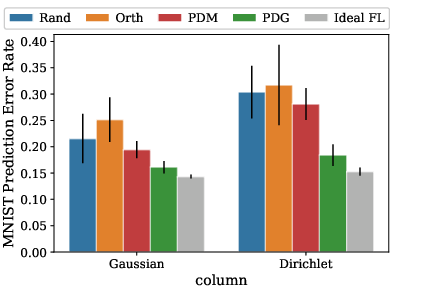}
    \vspace{-.5em}
    \caption{Federated MNIST classification performance averaged across 5 random realizations.
    Results are shown for both Gaussian (noisy data) and Dirichlet (imbalanced labels) non-i.i.d. scenarios.}
    \label{ff:non-iid-mnist}
\vspace*{-1em}
\end{figure}

\vspace*{-.5em}
\section{Conclusion}\label{s:end}

We developed a data-driven approach that can efficiently allocate transmit power for FL workers given channel instances and limited local information.  
The allocation is based on GCNs whose parameters are optimized through a primal-dual learning algorithm. 
Formal statements regarding the optimality of the proposed method were presented.
From a practical viewpoint, we successfully demonstrated the efficacy of the proposed method in (approximately) solving the associated non-convex constrained optimization problem.
Through extensive analyses and experiments, we have also highlighted additional appealing characteristics of the proposed approach such as robustness to the number of workers.
In addition, we have verified the proposed method in various challenging wireless and FL scenarios such as imperfect channel information and non-i.i.d. data at local workers.

In the future, there lies an intriguing avenue for modeling more variables, such as channel coherence time, worker mobility or velocity, and time-varying constraint bounds, which are often encountered in real-world scenarios. 
We are also interested in enhancing PD learning with adaptive step size when updating the variables and advanced graph network architectures.
Incorporating these additional dimensions would enrich our resource management capabilities for wireless FL systems and applications.

\vspace{-.5em}
\appendices
\section{}\label{ap:notations}

\begin{spacing}{1}
\begin{longtblr}[
  caption = \revmajor{Important notation.},
  label = {tab:notation},
]{
  colspec = {rp{188pt}},
  rowhead = 1,
} 
\toprule
Notation & Description\\
\hline
$L$    &  Number of workers in the system.\\
$i$    &  Index of workers, $\forall\,i=1,...,L$.\\
$k_i$  &  Number of data samples in worker $i$.\\
$K$    &  Total number of data samples in all workers. \\
$\omega_i$ &  Weight of worker $i$, characterizing how important it is compared to other workers. \\
$\bbw_i$  &  Local FL model parameters in worker $i$. \\
$\bbw$    &  Global FL model parameters in the base station (BS). \\
$\nabla$  &  Gradient operator. \\
$\bbH$ &  Channel-state information (CSI) matrix of size $L\times L$.\\
$\alpha_i$ &  Channel gain, i.e., the $i^\text{th}$ diagonal element of $\bbH$.\\
$\beta_{i,j}$ &  Interference coefficient, i.e., the element at $i^\text{th}$ row and $j^\text{th}$ column in $\bbH$.\\
$p$ &  Power allocation policy that decides how much power is allocated to perform the transmission in every worker.\\
$\bbp$ &  Transmit power vector of length $L$, which can be determined by $p$.\\
$P_{\max}$  &  Maximum transmit power allowed, which is a constant scalar value.\\
Link $i$ &  Transmission link from worker $i$ to BS.\\
$R_i$    &  Maximum data rate of link $i$.\\
$Z(\bbw_i)$ &  Data size to be transmitted through link $i$. \\
$S_i$ & $=$\(\left\{\begin{array}{rl}
0,  & \text{if the received data fails CRC} \\
1,  & \text{otherwise} \end{array} \right.\)\\
$\tau_i$ &  Transmission time from worker $i$ to BS.\\
$e_{\tot,i}$ & Total energy that worker $i$ consumes.\\
$r_0$, $e_0$ & Minimum constraint constants with regard to delay and energy requirements, respectively. \\
$\bbr$, $\bbe$, $\bby$ & Augmented primal variables that constrain the conditional expectations of data rate and energy efficiency, as well as the ergodic average of packet success rate (PSR), respectively.\\
$\bblambda_{\{r,e,y\}}$ & Corresponding dual variables introduced by the dual problem formulation. \\
$\gamma_{\{r,e,y\}}$ & Predefined step sizes to update the corresponding variables in the PD algorithm.\\
$f_{\{r,e,y\}}$ & Functions that compute worker-wise data rate, energy efficiency, and PSR values, respectively.\\
$\Psi$, $\psi_t$ & A $T$-layer GCN network and its $t^\text{th}$ layer, $\forall\,t=1,...,T$. \\
$\bbTheta^{(t)}$ & Network parameters in the $t^\text{th}$ layer. \\
$\bbZ^{(t)}$ & Output of the $t^\text{th}$ layer. \\
$d_{\{0,1,...,T\}}$ & Feature dimensions: $d_0$ for input, $d_{\{1,...,T-1\}}$ for hidden, and $d_T$ for output.\\
$\varphi$ & Non-linear activation function.\\
\bottomrule
\end{longtblr}
\end{spacing}

\section{Proof of Lemma~\ref*{L:1}}\label{ap:a}

Consider $\bbxi^1, \bbxi^2{\,\in\,}\ccalC$ with corresponding $p^1$ and $p^2$, respectively, such that the transmit power vectors are $\bbp^{1}{\,=\,}p^{1}(\bbH)$ and $\bbp^{2}{\,=\,}p^{2}(\bbH)$.
To show that $\ccalC$ is convex, we must show that $\alpha\bbxi^1 + (1 - \alpha)\bbxi^2 \in \ccalC$ for any $\alpha{\,\in\,}(0,1)$. 
Equivalently, for every $\alpha$ we need to define a power policy $p^\alpha{\,\in\,}\ccalP$ such that\looseness=-1
\begin{align*}
    v(p^\alpha)&\geq \alpha v(p^1) + (1-\alpha)v(p^2), \text{ and}\\
    u_i(p^\alpha)&\geq \alpha u_i(p^1) + (1-\alpha)u_i(p^2) + c_i, \forall i.
\end{align*}
Define a probability space $(\ccalH, \mbB(\ccalH), \mu)$, where $\ccalH{\,\in\,}\mbR^{L{\times}L}$ is the set of possible CSI matrices $\bbH$, $\mbB(\ccalH)$ is the $\sigma$-algebra of the Borel sets of $\ccalH$, and $\mu$ is a probability measure. 
For any $\ccalB{\,\in\,}\mbB(\ccalH)$, $i \in \{1, \ldots, L\}$ and $j \in \{1, 2\}$, we define the following quantity 
\begingroup
\begin{equation}\label{e:w_e}
E_{i}^{j}(\ccalB) =
\begin{cases}
  0 & \text{if } \rho_{i}^{j}=0,\\
  \frac{1}{\rho_{i}^{j}}\int_{\ccalB\cap\ccalH_i^j} {\fc}_i(\bbp^{j},\bbH)d\mu & \text{otherwise},
\end{cases}       
\end{equation}
\endgroup
where 
\begingroup
\begin{equation*}
    \rho_{i}^{j}{\,=\,}\int_{\ccalH}\!\mathds{1}(p_i^j{\,>\,}0)d\mu
\end{equation*}
\endgroup
and $\ccalH_i^j{\,=\,}\{\bbH\,|\,p_i^j(\bbH){\,>\,}0\}$.
Then, let us define the vector measure
\begingroup
\abovedisplayskip=4pt
\begin{equation}\label{e:w_def}
    \bbm(\ccalB)=[E_{1}^{1}(\ccalB), \hdots, E_{L}^{1}(\ccalB),
    E_{1}^{2}(\ccalB), \hdots, E_{L}^{2}(\ccalB)]^\top.
\end{equation}
\endgroup
{Also note that the element defined in~\eqref{e:w_e} is indexed by $\ind{i,j}{\,=\,}(j{\,-\,}1)L{\,+\,}i$ in $\bbm(\ccalB)$.}

\vspace{3mm}

\begin{claim}
The function $\bbm$ is a non-atomic measure on $\mbB(\ccalH)$.
\end{claim}
\begin{claimproof}
{
Let us start by showing that $\bbm$ is a valid measure.
The null empty set and non-negativity of $\bbm$ immediately follow from the definition in~\eqref{e:w_e}.
As for the countable disjoint additivity, under the condition of non-zero $\rho_i^j$ values -- otherwise it is apparent --, and adopting the notation ${\fc}_i^j = {\fc}_i(\bbp^{j},\bbH)$ we have
\begingroup
\allowdisplaybreaks
\begin{align*}
&\left[\sum_k\bbm(\ccalB_k)\right]_{\ind{i,j}} 
    = \frac{1}{\rho_{i}^{j}}\sum_k\int_{\ccalB_k\cap\ccalH_i^j} {\fc}_i(\bbp^{j},\bbH) d\mu \\ 
    & =
    \frac{1}{\rho_{i}^{j}}\int_{(\cup\ccalB_k)\cap\ccalH_i^j} {\fc}_i(\bbp^{j},\bbH) d\mu = \left[\bbm(\cup\ccalB_k)\right]_{\ind{i,j}},
\end{align*}
\endgroup
where the second equality follows from the fact that $\{\ccalB_k\}$'s are disjoint.
}

On the non-atomicity of $\bbm$, assume otherwise. 
If $\bbm$ is atomic, there exists an atom $\ccalA{\,\in\,}\mbB(\ccalH)$ with $\bbm(\ccalA){\,=\,}\bbepsilon$ with at least one strictly positive element, and any $\ccalA_S{\,\subseteq\,}\ccalA$ would imply either $\bbm(\ccalA_S){\,=\,}\bbzero$ or $\bbm(\ccalA_S){\,=\,}\bbepsilon$.
For any qualified $\ccalA$, we know that $\rho_{i}^{j}{\,>\,}0$ and $\ccalA{\,\cap\,}\ccalH_i^j{\,\neq\,}\varnothing$.
To reject the contradiction, all we need to show is that there exists a subset $\ccalA_S{\,\subset\,}\ccalA$ with complement $\ccalA_S^\complement{\,=\,}\ccalA{\setminus}\ccalA_S$ that satisfies $\ccalA_S{\,\cap\,}\ccalH_i^j{\,\neq\,}\varnothing$ and $\ccalA_S^\complement{\,\cap\,}\ccalH_i^j{\,\neq\,}\varnothing$, {which follows from the non-atomicity of measure $\mu$ on $\ccalH$, as assumed in \hyperref[as:1]{(AS1)}.}
\end{claimproof}

Notice that
\begingroup
\begin{equation*}
\bbm(\varnothing){\,=\!}\left[\begin{matrix}0\\
\vdots\\
0\\
0\\
\vdots\\
0\end{matrix}\right], \text{ and }
\bbm(\ccalH){\,=\!}\left[
\begin{matrix}\mbE[{\fc}_{1}(\bbp^{1},\bbH){\,|\,}p_{1}^{1}{\,>\,}0]\\
\vdots\\
\mbE[{\fc}_{L}(\bbp^{1},\bbH){\,|\,}p_{L}^{1}{\,>\,}0]\\
\mbE[{\fc}_{1}(\bbp^{2},\bbH){\,|\,}p_{1}^{2}{\,>\,}0]\\
\vdots\\
\mbE[{\fc}_{L}(\bbp^{2},\bbH){\,|\,}p_{L}^{2}{\,>\,}0]
\end{matrix}\right].        
\end{equation*}
\endgroup
The Lyapunov convexity theorem~\cite{nla.cat-vn1158105}, enabled by \hyperref[as:1]{(AS1)}, ensures that the range of $\bbm$ is convex.
Hence, for any $\alpha{\,\in\,}[0,1]$ we have that 
\begingroup
\abovedisplayskip=2pt
\begin{equation*}
    \bbm_0 = \alpha\bbm(\ccalH)+(1-\alpha)\bbm(\varnothing)
\end{equation*}
\endgroup
belongs to the range of $\bbm$, i.e., there exists some measurable set $\ccalB_0$ such that 
\begin{equation*}
    \bbm(\ccalB_0) = \alpha\bbm(\ccalH),  \text{ and }
    \bbm(\ccalB_0^\complement) = (1-\alpha)\bbm(\ccalH),
\end{equation*}
where $\ccalB_0^\complement = \ccalH \setminus \ccalB$.
Thus, let us define a power policy $p^\alpha{\,\in\,}\ccalP$:
\begingroup
\begin{equation*}
p^\alpha(\bbH) = \Big\{\,
\begin{matrix}
    p^1(\bbH) \text{, if } \bbH\in\ccalB_0,\\
    p^2(\bbH) \text{, if } \bbH\in\ccalB_0^\complement.
\end{matrix}
\end{equation*}
\endgroup
Let us compute
\begingroup
\allowdisplaybreaks
\abovedisplayskip=0pt
\begin{align*}
\mbE&\left[{\fc}_i(\bbp^{\alpha},\bbH){\,|\,}p_{i}^{\alpha}{\,>\,}0\right]
= \frac{\int_{\ccalH_{i}^{\alpha}}{\fc}_i(\bbp^{\alpha},\bbH)d\mu}
{\int_{\ccalH}\mathds{1}(p_i^{\alpha}>0)d\mu}\\
&= \frac{\int_{\ccalB_0\cap\ccalH_{i}^{1}}{\fc}_i(\bbp^{1},\bbH)d\mu + \int_{\ccalB^\complement_0\cap\ccalH_{i}^{2}}{\fc}_i(\bbp^{2},\bbH)d\mu}{\int_{\ccalB_0}\!\mathds{1}(p_i^1{\,>\,}0)d\mu{\,+\!}\!\int_{\ccalB_0^\complement}\!\mathds{1}(p_i^2{\,>\,}0)d\mu}\\
&= \frac{\int_{\ccalB_0\cap\ccalH_{i}^{1}}{\fc}_i(\bbp^{1},\bbH)d\mu}{\rho_i^1} +  
\frac{\int_{\ccalB^\complement_0\cap\ccalH_{i}^{2}}{\fc}_i(\bbp^{2},\bbH)d\mu}{\rho_i^2},
\end{align*}
\endgroup
where we have used \hyperref[as:3]{(AS3)} for the last equality.
Recalling the definition in~\eqref{e:w_e}, it then follows that\looseness=-1
\begingroup
\allowdisplaybreaks
\begin{align*}
\mbE&\left[{\fc}_i(\bbp^{\alpha},\bbH){\,|\,}p_{i}^{\alpha}{\,>\,}0\right]
= E_{i}^{1}(\ccalB_0) + E_{i}^{2}(\ccalB_0^\complement) 
 \\
 &= \alpha \mbE[{\fc}_i(\bbp^{1},\bbH){\,|\,}p_{i}^{1}{\,>\,}0] + (1-\alpha)\mbE[{\fc}_i(\bbp^{2},\bbH){\,|\,}p_{i}^{2}{\,>\,}0]\\
&\geq \alpha(\xi^{1}_{i} + c_i) + (1-\alpha) (\xi^{2}_{i} + c_i) = \left(\alpha\xi^{1}_{i} + (1-\alpha) \xi^{2}_{i}\right) + c_i.
\end{align*}
\endgroup
Moreover, we can compute $v(p^\alpha)$ to obtain
\begingroup
\allowdisplaybreaks
\begin{align*}
    & v(p^\alpha) = g\!\left(\mbE\left[f_0(\bbp^{\alpha},\bbH)\right]\right) = g\!\left( \int_\ccalH f_0(\bbp^\alpha,\bbH)d\mu \right)\\
    & = g\!\left(\int_{\ccalB_0}f_0(\bbp^{1},\bbH)d\mu + \int_{\ccalB_{0}^{\complement}}f_0(\bbp^{2},\bbH)d\mu\right)\\
    &= g\!\left( \alpha\mbE[f_0(\bbp^{1},\bbH)] + (1-\alpha)\mbE[f_0(\bbp^{2},\bbH)] \right),
\end{align*}
\endgroup
where the last equality follows, again, from the Lyapunov convexity theorem given by a non-atomic measure similar to what was done with $\bbm$ [c.f.~\eqref{e:w_def}].

Due to the concavity of $g$ in \hyperref[as:2]{(AS2)} we have that 
\begin{equation*}
    v(p^\alpha)\geq \alpha v(p^1) + (1-\alpha) v(p^2) 
    \geq \alpha \bbxi^1 + (1-\alpha)\bbxi^2.
\end{equation*}

This shows that $\alpha \bbxi^1 + (1-\alpha)\bbxi^2{\,\in\,}\ccalC$, as wanted. 

\vspace*{-1em}
\section{Proof of Theorem~\ref*{T:2}}\label{ap:b}

Setting aside the parameterization of the power policy function for now, let us revisit the unparameterized problem but incorporating variables $\bby$ and $\bbc$
\begingroup
\allowdisplaybreaks
\begin{subequations*}
    \begin{alignat*}{3}
        P^\star &= \max\limits_{p,\bby,\bbc}\,\, g(\bby),\\
        \text{s.t.} \quad & {\bbz}(p)\geq \bby, \,\, u_i(p) \geq c_i,\,\, \forall\,i, \\
         & \bby \geq \bbzero, \bbc{\,\in\,}[\bbc_{0},+\infty),\,\, 
         p{\,\in\,}\ccalP, \,\,
         \forall\,\bbH,
    \end{alignat*}
\end{subequations*}
\endgroup
where ${\bbz}(p){\,=\,}\mbE{\left[\,f_0(\bbp,\bbH)\,\right]}$ is the elementwise expectation of the objective function, and 
${u}_i(p)=\mbE{\left[{\fc}_i(\bbp,\bbH) \, | \, {p}_i\!> \!0\right]}$ is the conditional expectation of constraints, with $\bbp{\,=\,}p(\bbH)$ as the allocated power. 
Its Lagrangian is as follows
\begin{equation*}
    \ccalL(p,\bby,\bbc,
    \bblambda_y,\bblambda_c ) = g(\bby)
    + \bblambda_y^\top (\bbz(p) - \bby) 
    + \bblambda_c^\top (\bbu(p) - \bbc).
\end{equation*}

Using this notation, we have 
\begin{equation*}
D^\star = \min\limits_{\bblambda_y, \bblambda_c \geq \bbzero}\,\, \max\limits_{p\in\ccalP, \bby{\geq}\bbzero, \bbc{\geq}\bbc_0}\ccalL(p,\bby, \bbc,\bblambda_y, \bblambda_c).
\end{equation*}

With the parameterization of the power allocation function using a deep neural architecture, we have $\bbp_{\psi}{\,=\,}p_{\psi}(\bbH;\bbTheta)$ and
\begingroup
\medmuskip=0mu
\thinmuskip=0mu
\thickmuskip=0mu
\nulldelimiterspace=0pt
\scriptspace=0.5pt
\belowdisplayskip=0pt
\begin{align}\label{e:D_psi_star}
    D_{\psi}^\star = \min\limits_{\bblambda_y, \bblambda_c\geq \bbzero} \Big\{
    \max\limits_{\bby{\geq}\bbzero, \bbc{\geq}\bbc_0}\,
    v_p(\bby, \bbc, \bblambda_y, \bblambda_c) 
    +\max\limits_{\bbTheta\in \ccalP_{\psi}}\,
    v_d(\bblambda_y, \bblambda_c,\bbp_{\psi})\Big\},
\end{align}
\endgroup
where 
\begingroup
\abovedisplayskip=0pt
\begin{equation*}
    \begin{gathered}
        v_p(\bby, \bbc, \bblambda_y, \bblambda_c)=g(\bby) - \bblambda_y^\top\bby - \bblambda_c^\top\bbc,\\
        v_d(\bblambda_y, \bblambda_c,\bbp_{\psi}) = \bblambda_y^\top\mbE[f_0(\bbp_{\psi},\bbH)]
        +\bblambda_c^\top\mbE[f_c(\bbp_{\psi},\bbH){\,|\,}S],
    \end{gathered}
\end{equation*}
\endgroup
with $S$ denoting the event of worker selection. 
Only those workers whose transmit power is greater than zero can attempt to transmit their local parameters.
From \hyperref[as:3]{(AS3)} it follows that $\mathds{1}(\bbp{\,>\,}0){\,=\,}\mathds{1}(\bbp^{\star}{\,>\,}0)$ for any feasible allocation $\bbp$, where $\bbp^\star$ corresponds to the optimal policy.
Putting it differently, the conditional expectation of the constraint(s) will be always conditioned on the same event $S$.
Subsequently, we have that 
\begingroup
\allowdisplaybreaks
\begin{align*}
    \max\limits_{\bbTheta\in\ccalP_{\psi}}&\,v_d(\bblambda_y, \bblambda_c, \bbp_{\psi}) = v_d(\bblambda_y, \bblambda_c,\bbp^{\star})\\
    &+\max\limits_{\bbTheta\in\ccalP_{\psi}}\Big\{ \bblambda_y^\top\mbE[f_0(\bbp_{\psi},\bbH) - f_0(\bbp^{\star},\bbH)]\nonumber\\
    &+\bblambda_c^\top\mbE[f_c(\bbp_{\psi},\bbH)
    -f_c(\bbp^{\star},\bbH){\,|\,}S]\Big\}.\nonumber
\end{align*}
\endgroup

Since the parameterized power policy space $\ccalP_{\psi}{\,\in\,}\ccalP$, it must be that the maximized objective  $g(\bby_{\psi}^\star){\,\leq\,}g(\bby^\star)$.
Following the monotonicity of $g$ in \hyperref[as:2]{(AS2)} and the equality of the ergodic constraints due to optimality, the following inequalities should hold:
\begin{alignat*}{3}
    &\E{f_0(\bbp_{\psi},\bbH)} 
    &= \bby_{\psi}^{\star} \leq \bby^\star 
    &= \E{f_0(\bbp^\star, \bbH)}.
\end{alignat*}
On the conditional expectation of the constraint function, accounting for~\hyperref[as:4]{(AS4)}, we may apply H{\"o}lder's inequality~\cite{maligranda1998holder} and have 
\begingroup
\abovedisplayskip=2pt
\begin{align*}
    & \bblambda_c^\top\mbE[f_c(\bbp_{\psi},\bbH)
    -f_c(\bbp^{\star},\bbH){\,|\,}S] \\
    \leq &\|\bblambda_c\|_1 \|\E{f_c(\bbp_{\psi}^{\star},\bbH) - f_c(\bbp^{\star},\bbH)\,|\,S}\|_\infty\numberthis\label{e:up_bound_1}
\end{align*}
\endgroup
which, due to the convexity of the infinity norm and the continuity of the expectation by~\hyperref[as:4]{(AS4)}, can be further bounded as\looseness=-1
\begingroup
\begin{align}
    \eqref{e:up_bound_1} \leq L\|\bblambda_c\|_1 \|\bbp_{\psi}^{\star}-\bbp^{\star}\|_\infty 
    \leq L\epsilon \|\bblambda_c\|_1,\label{e:up_bound_2}
\end{align}
\endgroup
where the last inequality comes from the expressivity of the parameterization quantified by~\hyperref[as:5]{(AS5)}.
Meanwhile, the non-negativity of $\bblambda_y$ and $\bblambda_c$ gives that
\begingroup
\begin{align*}
   \max\limits_{\bbTheta\in\ccalP_{\psi}}\,v_d(\bblambda_y, \bblambda_c,\bbp_{\psi}) \leq v_d(\bblambda_y, \bblambda_c,\bbp^\star) + L\epsilon \|\bblambda_c\|_1.
\end{align*}
\endgroup
Thus, we have the following inequality: 
\begingroup
\allowdisplaybreaks
\medmuskip=1mu
\thinmuskip=1mu
\thickmuskip=1mu
\nulldelimiterspace=0pt
\scriptspace=0pt
\begin{align*}
    &D_{\psi}^\star \leq \min\limits_{\bblambda_y, \bblambda_c \geq \bbzero} \Big\{\max_{\substack{\bby\geq\bbzero\\\bbc\geq\bbc_0}}\,v_p(\bby, \bbc, \bblambda_y, \bblambda_c)+v_d(\bblambda_y, \bblambda_c,\bbp^\star) + L\epsilon \|\bblambda_c\|_1\Big\}\\
    &=\min\limits_{\bblambda_y, \bblambda_c \geq \bbzero}\Big\{\max_{\substack{\bby\geq\bbzero\\\bbc\geq\bbc_0}}\,v_p(\bby, \bbc, \bblambda_y, \bblambda_c)+\max_{p\in \ccalP}\,v_d(\bblambda_y, \bblambda_c,\bbp)  + L\epsilon \|\bblambda_c\|_1\Big\} \\
    &=\,\,D^\star\,+\,L\epsilon \|\bblambda_c^\star\|_1,
\end{align*}
\endgroup
in which $D^\star{\,=\,}P^\star$ due to~\Cref{T:1} and stands for the upper bound in~\Cref{T:2}.
This leaves us the lower bound to prove in the remaining.

For all $p{\,\in\,}\ccalP$, it holds that
\begingroup
\allowdisplaybreaks
\begin{align*}
    &\max_{\bbTheta\in\ccalP_{\psi}}\,v_d(\bblambda_y, \bblambda_c,\bbp_{\psi}) = -\min\limits_{\bbTheta\in\ccalP_{\psi}}\,-v_d(\bblambda_y, \bblambda_c,\bbp_{\psi})\\
    &= v_d(\bblambda_y, \bblambda_c,\bbp) -[v_d(\bblambda_y, \bblambda_c,\bbp) + \min\limits_{\bbTheta\in\ccalP_{\psi}}-v_d(\bblambda_y, \bblambda_c,\bbp_{\psi})]\\
    &= v_d(\bblambda_y, \bblambda_c,\bbp)
    -\min\limits_{\bbTheta\in\ccalP_{\psi}}\big\{\bblambda_y^\top\mbE[
    f_0(\bbp,\bbH)- f_0(\bbp_{\psi},\bbH)]\\
    &\hspace*{68pt}+\bblambda_c^\top\mbE[f_c(\bbp,\bbH) - f_c(\bbp_{\psi},\bbH)\,|\,S]\big\}.\numberthis\label{e:lowb_eq1}
\end{align*}  
\endgroup

Following an approach similar to that used in~\eqref{e:up_bound_1} --~\eqref{e:up_bound_2}, we have that 
\begingroup
\allowdisplaybreaks
\begin{align*}
    \min\limits_{\bbTheta\in\ccalP_{\psi}}\{\bblambda_y^\top\mbE[f_0(\bbp,\bbH)&- f_0(\bbp_{\psi},\bbH)] \\
    &+ \bblambda_c^\top\E{f_c(\bbp,\bbH) - f_c(\bbp_{\psi},\bbH)\,|\,S}\}\\
    &\hspace*{-80pt}\leq\|\bblambda_y\|_1 \min\limits_{\bbTheta\in\ccalP_{\psi}}\|\E{f_0(\bbp,\bbH) - f_0(\bbp_{\psi},\bbH)}\|_\infty\\
    &\hspace*{-50pt}+\|\bblambda_c\|_1 \min\limits_{\bbTheta\in\ccalP_{\psi}}\|\E{f_c(\bbp,\bbH) - f_c(\bbp_{\psi},\bbH)\,|\,S}\|_\infty\\
    &\hspace*{-80pt}\leq \min\limits_{\bbTheta\in\ccalP_{\psi}} \left\{ L\|\bbp-\bbp_{\psi}\|_\infty\right\}\left(\|\bblambda_y\|_1 + \|\bblambda_c\|_1\right)\\
    &\hspace*{-80pt}\leq L\epsilon \left(\|\bblambda_y\|_1 + \|\bblambda_c\|_1\right),\numberthis\label{e:ineq_conti}
\end{align*} 
Then, substituting~\eqref{e:ineq_conti} from~\eqref{e:lowb_eq1}, we get 
\begin{align*}
    \eqref{e:lowb_eq1} \geq v_d(\bblambda_y, \bblambda_c,\bbp) - L\epsilon \left(\|\bblambda_y\|_1 + \|\bblambda_c\|_1 \right).
\end{align*} 
Following the definition~\eqref{e:D_psi_star}, we get
\begin{align}\label{e:d_psi_lb}
    D_{\psi}^{\star} \geq \min\limits_{\bblambda_y, \bblambda_c\geq \bbzero} \Big\{&
    \max_{\bby\geq\bbzero, \bbc\geq\bbc_0}\,v_p(\bby, \bbc, \bblambda_y, \bblambda_c)\\
    &+v_d(\bblambda_y, \bblambda_c,\bbp)- L\epsilon \left(\|\bblambda_y\|_1 + \|\bblambda_c\|_1\right)\Big\}.\nonumber
\end{align}
From $\bblambda_y,\bblambda_c{\,\geq\,}\bbzero$, it follows that $\|\bblambda_y\|_1{\,=\,}\bblambda_y^\top\bbone$, same for others. 
And, since the lower bound of $D_{\psi}^{\star}$ in~\eqref{e:d_psi_lb} holds for all $p{\,\in\,}\ccalP$, we can rewrite it as follows:
\begingroup
\begin{align}
    D_{\psi}^{\dag\star} &\geq \min\limits_{\bblambda_y, \bblambda_c\geq \bbzero} \Big\{
    \max_{\bby\geq\bbzero, \bbc\geq\bbc_0}v_p(\bby, \bbc, \bblambda_y, \bblambda_c)
    +\max_{p\in\ccalP}\Big[ \\
    &\bblambda_y^\top(\E{f_0(\bbp,\bbH)}{-}L\epsilon\bbone ) 
    +\bblambda_c^\top(\E{f_c(\bbp,\bbH)\,|\,S}{-}L\epsilon\bbone )\Big]\Big\},\nonumber
\end{align}
\endgroup
whose right-hand side is the dual value of the $(L\epsilon)$-perturbed version of the original problem~\eqref{e:p2g}, as stated in the following:
\begingroup
\abovedisplayskip=-2pt
\begin{subequations}
    \begin{alignat*}{3}
        P_{L\epsilon}^{\star}&=\max_{p,\bby,\bbc}\quad  g(\bby), \\
        \text{s.t.} \,\, & L\epsilon\bbone+\bby \leq \E{f_0(\bbp,\bbH)},\\
         & L\epsilon\bbone+\bbc \leq \E{f_c(\bbp,\bbH)\,|\,S},\\
         & p\in\ccalP, \bby\geq\bbzero, \bbc\geq\bbc_0,
    \end{alignat*}
\end{subequations}
\endgroup
whose strong duality has been shown by~\cite[Appendix~A]{eisen2019learning}, implying that $D_{\psi}^{\star}{\,\geq\,}D_{L\epsilon}^{\star}{\,=\,}P_{L\epsilon}^{\star}$.
According to the perturbation inequality that relates $P_{L\epsilon}^{\star}$ to $P^{\star}$:
\begingroup
\begin{equation*}
    P_{L\epsilon}^{\star}\geq P^{\star}-\left(\|\bblambda_y^\star\|_1 + \|\bblambda_c^\star\|_1\right)L\epsilon,
\end{equation*}
\endgroup
we have the lower bound on $D_{\psi}^{\star}$ as stated.
\vspace{-.5em}

\bibliographystyle{IEEEbib}
{
\linespread{.99}
\small\bibliography{IEEEabrv,references}
}

\end{document}